\documentclass{article}

\usepackage[preprint]{neurips_2020}

\usepackage[utf8]{inputenc} % allow utf-8 input
\usepackage[T1]{fontenc}    % use 8-bit T1 fonts
\usepackage{hyperref}       % hyperlinks
\usepackage{url}            % simple URL typesetting
\usepackage{booktabs}       % professional-quality tables
\usepackage{amsfonts}       % blackboard math symbols
\usepackage{nicefrac}       % compact symbols for 1/2, etc.
\usepackage{microtype}      % microtypography
\usepackage{changepage} 

\usepackage{algorithm}
\usepackage{algorithmic}
\usepackage{todonotes}

\usepackage{natbib}
\usepackage{amssymb}
\usepackage{amsmath}
\usepackage{amsthm}

\title{Variational Model-based Policy Optimization}

\author{%
  \begin{tabular}{c@{\hspace*{.8cm}}c}
  {\bf Yinlam Chow} & {\bf Brandon Cui} \\[.1cm]
  {\normalfont Google Research} & {\normalfont Facebook AI Research} \\[.1cm]
  {\normalfont yinlamchow@google.com} & {\normalfont bcui@fb.com} \\[.5cm]
   {\bf Moonkyung Ryu} & {\bf Mohammad Ghavamzadeh}\\[.1cm]
  {\normalfont Google Research} & {\normalfont Google Research}\\[.1cm]
  {\normalfont mkryu@google.com} & {\normalfont ghavamza@google.com}\\[.1cm]
  \end{tabular}
}

\usepackage{color}
\usepackage[normalem]{ulem}

\newcommand{\Bo}[1]{}
\newcommand{\comment}[1]{}

\newtheorem{theorem}{Theorem}
\newtheorem{lemma}[theorem]{Lemma}

\newtheorem{corollary}[theorem]{Corollary}

\def\eqref#1{equation~\ref{#1}}
\def\1{\bm{1}}

\DeclareMathAlphabet{\mathsfit}{\encodingdefault}{\sfdefault}{m}{sl}
\SetMathAlphabet{\mathsfit}{bold}{\encodingdefault}{\sfdefault}{bx}{n}

\DeclareMathOperator*{\argmax}{arg\,max}
\DeclareMathOperator*{\argmin}{arg\,min}

\newcommand{\control}{a}

\newcommand{\obs}{x}

\renewcommand{\widehat}{\hat}
\newtheorem{remark}{Remark}
\newtheorem{proposition}{Proposition}

\begin{document}

\maketitle

\begin{abstract}
Model-based reinforcement learning (RL) algorithms allow us to combine model-generated data with those collected from interaction with the real system in order to alleviate the data efficiency problem in RL. However, designing such algorithms is often challenging because the bias in simulated data may overshadow the ease of data generation. A potential solution to this challenge is to jointly learn and improve model and policy using a universal objective function. In this paper, we leverage the connection between RL and probabilistic inference, and formulate such an objective function as a variational lower-bound of a log-likelihood. This allows us to use expectation maximization (EM) and iteratively fix a baseline policy and learn a variational distribution, consisting of a model and a policy (E-step), followed by improving the baseline policy given the learned variational distribution (M-step). We propose model-based and model-free policy iteration (actor-critic) style algorithms for the E-step and show how the variational distribution learned by them can be used to optimize the M-step in a fully model-based fashion. Our experiments on a number of continuous control tasks show that despite being more complex, our model-based (E-step) algorithm, called {\em variational model-based policy optimization} (VMBPO), is more sample-efficient and robust to hyper-parameter tuning than its model-free (E-step) counterpart. Using the same control tasks, we also compare VMBPO with several state-of-the-art model-based and model-free RL algorithms and show its sample efficiency and performance. 
% While model-based algorithms alleviates the prevalent data inefficiency issue in reinforcement learning (RL), designing effective algorithms is oftentimes challenging because the model-generated data may contain bias (from the  dynamics model which is learned independently of RL) that may hinder policy learning. In this paper, we introduce a new algorithm that jointly learn the model and policy based on a universal RL objective. Leveraging the connection between reinforcement learning and probabilistic inference, we propose a novel probabilistic inference framework for RL that models both the dynamics model and policy. Applying variational expectation-maximization (EM) to this framework, we thus show that the model-based actor-critic algorithm can be reduced to an iterative EM algorithm, with variational model and policy learning equivalent to an E-step and policy improvement to an M-step. This algorithm has the main advantages of effective data generation and improved policy performance both due to the learned dynamics model that is optimized for RL. To demonstrate the effectiveness of our algorithm, we compare it with several state-of-the-art baselines on continuous RL experiments and show that our method has better sample-efficiency, faster convergence, and more robustness to hyper-parameter settings.
\end{abstract}
\vspace{-0.15in}

%%%%%%%%%%%%%%%%%%%%%%%%%%%%%%%%%%%%%%%%%%%%%%%%%%%%%%%%%%%%%%%%%%%%%%%%%%
%%%%%%%%%%%%%%%%%%%%%%%%%%%%%%%%%%%%%%%%%%%%%%%%%%%%%%%%%%%%%%%%%%%%%%%%%%
%%%%%%%%%%%%%%%%%%%%%%%%%%%%%%%%%%%%%%%%%%%%%%%%%%%%%%%%%%%%%%%%%%%%%%%%%%
%%%%%%%%%%%%%%%%%%%%%%%%%%%%%%%%%%%%%%%%%%%%%%%%%%%%%%%%%%%%%%%%%%%%%%%%%%
%%%%%%%%%%%%%%%%%%%%%%%%%%%%%%%%%%%%%%%%%%%%%%%%%%%%%%%%%%%%%%%%%%%%%%%%%%

\vspace{-0.1in}
\section{Introduction}
\label{sec:intro}
\vspace{-0.1in}

Model-free reinforcement learning (RL) algorithms that learn a good policy without constructing an explicit model of the system's dynamics have shown promising results in complex simulated problems~\citep{Mnih13PA,Mnih15HL,Schulman15TR,Haarnoja18SA}. However, these methods are not sample efficient, and thus, not suitable for problems in which data collection is burdensome. Model-based RL algorithms address the data efficiency issue of the model-free methods by learning a model, and combining model-generated data with those collected from interaction with the real system~\citep{Janner19WT}. However, designing model-based RL algorithms is often challenging because the bias in model may affect the process of learning policies and result in worse asymptotic performance than the model-free counterparts. A potential solution to this challenge is to incorporate the policy/value optimization method in the process of learning the model~\citep{Farahmand18IV,Abachi20PA}. An ideal case here would be to have a universal objective function that is used to learn and improve model and policy jointly. 

Casting RL as a probabilistic inference has a long history~\citep{Todorov08GD,Toussaint09RT,Kappen12OC,Rawlik13SO}. This formulation has the advantage that allows powerful tools for approximate inference to be employed in RL. One such class of tools are variational techniques~\citep{Hoffman13VI} that have been successfully used in RL~\citep{Neumann11VI,Levine13VP,Abdolmaleki18MP}. Another formulation of RL with strong connection to probabilistic inference is the formulation of policy search as an expectation maximization (EM) style algorithm~\citep{Dayan97EM,Peters07RL,Peters10RE,Neumann11VI,Chebotar17PI,Abdolmaleki18MP}. The main idea here is to write the expected return of a policy as a (pseudo)-likelihood function, and then assuming success in maximizing the return, finding the policy that most likely would have been taken. Another class of RL algorithms that are related to the RL as inference formulation are entropy-regularized algorithms that add an entropy term to the reward function and find the soft-max optimal policy~\citep{Levine14LC,Levine14LN,Nachum17BG,Nachum18PC,Haarnoja18SA,fellows2019virel}. For a comprehensive tutorial on RL as probabilistic inference, we refer readers to~\cite{Levine18RL}. 

%Methods that have been developed based on formulating RL as probabilistic inference can also be divided to model-free and model-based. Methods such as~\cite{Todorov08GD,Toussaint09RT,Levine13VP,Chebotar17PI} that assume access to linearized dynamics, and those like~\cite{Kappen12OC} that use the overall dynamics belong to the model-based category. On the other hand, model-free methods include~\cite{Peters10RE,Neumann11VI,Daniel16HR}

In this paper, we leverage the connection between RL and probabilistic inference, and formulate an objective function for jointly learning and improving model and polciy as a variational lower-bound of a log-likelihood. This allows us to use EM, and iteratively fix a baseline policy and learn a variational distribution, consisting of a model and a policy (E-step), followed by improving the baseline policy given the learned variational distribution (M-step). We propose model-based and model-free policy iteration (PI) style algorithms for the E-step and show how the variational distribution that they learn can be used to optimize the M-step, only from model-generated samples. Both algorithms are model-based but they differ in using model-based and model-free algorithms for the E-step. 
Our experiments on a number of continuous control tasks show that although our algorithm that uses model-based PI for the E-step, which we call it {\em variational model-based policy optimization} (VMBPO), is more complex than its model-free counterpart, it is more sample-efficient and robust to hyper-parameter tuning. Using the same control tasks, we also compare VMBPO with several state-of-the-art model-based and model-free RL algorithms, including model-based policy optimization (MBPO)~\cite{Janner19WT} and maximum a posteriori policy optimization (MPO)~\cite{Abdolmaleki18MP}, and show its sample efficiency and performance.

%%%%%%%%%%%%%%%%%%%%%%%%%%%%%%%%%%%%%%%%%%%%%%%%%%%%%%%%%%%%%%%%%%%%%%%%%%
%%%%%%%%%%%%%%%%%%%%%%%%%%%%%%%%%%%%%%%%%%%%%%%%%%%%%%%%%%%%%%%%%%%%%%%%%%
%%%%%%%%%%%%%%%%%%%%%%%%%%%%%%%%%%%%%%%%%%%%%%%%%%%%%%%%%%%%%%%%%%%%%%%%%%
%%%%%%%%%%%%%%%%%%%%%%%%%%%%%%%%%%%%%%%%%%%%%%%%%%%%%%%%%%%%%%%%%%%%%%%%%%
%%%%%%%%%%%%%%%%%%%%%%%%%%%%%%%%%%%%%%%%%%%%%%%%%%%%%%%%%%%%%%%%%%%%%%%%%%

\vspace{-0.1in}
\section{Preliminaries}
\label{sec:prelim}
\vspace{-0.1in}

We study the reinforcement learning (RL) problem~\citep{Sutton18RL} in which the agent's interaction with the environment is modeled as a discrete-time Markov decision process (MDP) $\mathcal M = \langle \mathcal X, \mathcal A, r, p, p_0 \rangle$, where $\mathcal X$ and $\mathcal A$ are state and action spaces; $r:\mathcal X\times\mathcal A\rightarrow\mathbb R$ is the reward function; $p:\mathcal X\times\mathcal A\rightarrow\Delta_{\mathcal X}$ ($\Delta_{\mathcal X}$ is the set of probability distributions over $\mathcal X$) is the transition kernel; and $p_0:\mathcal X\rightarrow\Delta_{\mathcal X}$ is the initial state distribution. A stationary Markovian policy $\pi:\mathcal X\rightarrow\Delta_{\mathcal A}$ is a probabilistic mapping from states to actions. Each policy $\pi$ is evaluated by its {\em expected return}, i.e.,~$J(\pi) = \mathbb E[\sum_{t=0}^{T-1} r(x_t,a_t)\mid p_0,p,\pi]$, where $T$ is the {\em stopping time}, i.e.,~the random variable of hitting a {\em terminal state}.\footnote{Similar to~\citet{Levine18RL}, our setting can be easily extended to infinite-horizon $\gamma$-discounted MDPs. This can be done by modifying the transition kernels, such that any action transitions the system to a terminal state with probability $1-\gamma$, and all standard transition probabilities are multiplied by $\gamma$.} We denote by ${\mathcal X}^0$ the set of all terminal states. The agent's goal is to find a policy with maximum expected return, i.e,~$\pi^*\in\argmax_{\pi\in\Delta_{\mathcal A}} J(\pi)$. We denote by $\xi=(x_0,a_0,\ldots,x_{T-1},a_{T-1},x_T)$, a system trajectory of lenght $T$, whose probability under a policy $\pi$ is defined as $p_\pi(\xi)=p_0(x_0)\prod_{t=0}^{T-1}\pi(a_t|x_t)p(x_{t+1}|x_t,a_t)$. Finally, we define $[T] := \{0,\ldots,T-1\}$.

\vspace{-0.1in}
\section{Policy Optimization as Probabilistic Inference}
\label{sec:RL-Inference}
\vspace{-0.1in}

Policy search in reinforcement learning (RL) can be formulated as a probabilistic inference problem (e.g.,~\citealt{Todorov08GD,Toussaint09RT,Kappen12OC,Levine18RL}). The goal in the conventional RL formulation is to find a policy whose generated trajectories maximize the expected return. In contrast, in the inference formulation, we start with a prior over trajectories and then estimate the posterior conditioned on a desired outcome, such as reaching a goal state. In this formulation, the notion of a desired (optimal) outcome is introduced via {\em independent} binary random variables $\mathcal O_t,\;t\in[T]$, where $\mathcal O_t=1$ denotes that we acted optimally at time $t$. The likelihood of $\mathcal O_t$, given the state $x_t$ and action $a_t$, is modeled as $p(\mathcal O_t=1\mid x_t,a_t) = \exp(\eta\cdot r(x_t,a_t))$, where $\eta>0$ is a temperature parameter. This allows us to define the log-likelihood of a policy $\pi$ being optimal as 

\vspace{-0.15in}
\begin{small}
\begin{equation}
\label{eq:likelihood-policy}
\log p_\pi({\mathcal O}_{0:T-1}=1) = \log \int_\xi p_\pi({\mathcal O}_{0:T-1}=1,\xi) = \log\mathbb E_{\xi\sim p_\pi}\big[p({\mathcal O}_{0:T-1}=1\mid\xi)\big],
\end{equation}
\end{small}
\vspace{-0.15in}

where $p({\mathcal O}_{0:T-1}=1\mid\xi)$ is the optimality likelihood of trajectory $\xi$ and is defined as

\vspace{-0.15in}
\begin{small}
\begin{equation}
\label{eq:likelihood-traj}
p({\mathcal O}_{0:T-1}=1\mid\xi) = \prod_{t=0}^{T-1}p(\mathcal O_t=1\mid x_t,a_t) = \exp\big(\eta\cdot\sum_{t=0}^{T-1}r(x_t,a_t)\big).
\end{equation}
\end{small}
\vspace{-0.15in}

As a result, finding an optimal policy in this setting would be equivalent to maximizing the log-likelihood in Eq.~\ref{eq:likelihood-policy}, i.e.,~$\pi^*_{\text{soft}}\in\argmax_{\pi}\log p_\pi({\mathcal O}_{0:T-1}=1)$.

A potential advantage of formulating RL as an inference problem is the possibility of using a wide range of approximate inference algorithms, including variational methods. In variational inference,  %inference~\citep{Hoffman13VI}, 
we approximate a distribution $p(\cdot)$ with a potentially simpler (e.g.,~tractable factored) distribution $q(\cdot)$ in order to make the whole inference process more tractable. If we approximate $p_\pi(\xi)$ with a variational distribution $q(\xi)$, we will obtain the following variational lower-bound for the log-likelihood in Eq.~\ref{eq:likelihood-policy}:

\vspace{-0.175in}
\begin{small}
\begin{align}
\log p_\pi(&{\mathcal O}_{0:T-1}=1) = \log \mathbb E_{\xi\sim p_\pi}\big[\exp\big(\eta\cdot\sum_{t=0}^{T-1}r(x_t,a_t)\big)\big] = \log \mathbb E_{\xi\sim q(\xi)}\big[\frac{p_\pi(\xi)}{q(\xi)}\cdot \exp\big(\eta\cdot\sum_{t=0}^{T-1}r(x_t,a_t)\big)\big] \nonumber \\
&\stackrel{\text{(a)}}{\geq} \mathbb E_{\xi\sim q(\xi)}\big[\log\frac{p_\pi(\xi)}{q(\xi)} + \eta\sum_{t=0}^{T-1}r(x_t,a_t)\big] = \eta\cdot\mathbb E_q\big[\sum_{t=0}^{T-1}r(x_t,a_t)\big] - \text{KL}(q||p_\pi) := \mathcal J(q;\pi),
\label{eq:ELBO1}
\end{align}
\end{small}
\vspace{-0.125in}

{\bf (a)} is from Jensen's inequality, and $\mathcal J(q;\pi)$ is the evidence lower-bound (ELBO) of the log-likelihood function. A variety of algorithms have been proposed (e.g.,~\citealt{Peters07RL,Hachiya09ES,Neumann11VI,Levine13VP,Abdolmaleki18MP,fellows2019virel}), whose main idea is to approximate $\pi^*_{\text{soft}}$ by maximizing $\mathcal J(q;\pi)$ w.r.t.~both $q$ and $\pi$. This often results in an expectation-maximization (EM) style algorithm in which we first fix $\pi$ and maximize $\mathcal J(\cdot;\pi)$ for $q$ (E-step), and then for the $q$ obtained in the E-step, we maximize $\mathcal J(q;\cdot)$ for $\pi$ (M-step). %Variational policy search algorithms~\citep{Neumann11VI,Levine13VP} also belong to this category. 

\vspace{-0.1in}
\section{Variational Model-based Policy Optimization}
\label{sec:elbo}
\vspace{-0.1in}

In this section, we describe the ELBO objective function used by our algorithms, study the properties of the resulted optimization problem, and propose algorithms to solve it. We propose to use the variational distribution $q(\xi)=p_0(x_0)\prod_{t=0}^{T-1}q_c(a_t|x_t)q_d(x_{t+1}|x_t,a_t)$ to approximate $p_\pi(\xi)$. Note that $q$ shares the same initial state distribution as $p_\pi$, but has different control strategy (policy), $q_c$, and dynamics, $q_d$. Using this variational distribution, we may write the ELBO objective of (\ref{eq:ELBO1}) as

\vspace{-0.15in}
\begin{small}
\begin{equation}
\label{eq:ELBO2}
\mathcal J(q;\pi) = \mathbb E_q\big[\sum_{t=0}^{T-1} \eta\cdot r(x_t,a_t) - \log\frac{q_c(a_t|x_t)}{\pi(a_t|x_t)} - \log\frac{q_d(x_{t+1}|x_t,a_t)}{p(x_{t+1}|x_t,a_t)}\big], \;\; \text{where } \; \mathbb E_q[\cdot]:=\mathbb E[\cdot|p_0,q_d,q_c].
\end{equation}
\end{small}
\vspace{-0.15in}

% \vspace{-0.15in}
% \begin{small}
% \begin{align}
% \mathcal J(q;\pi) &= \eta\cdot\mathbb E_q\big[\sum_{t=0}^{T-1} r(x_t,a_t)\big] - \mathbb E_q\big[\log \prod_{t=0}^{T-1}\frac{ q_c(a_t|x_t)}{\pi(a_t|x_t)}\cdot\frac{q_d(x_{t+1}|x_t,a_t)}{p(x_{t+1}|x_t,a_t)}\big] \nonumber \\
% &= \mathbb E_q\big[\sum_{t=0}^{T-1} \eta\cdot r(x_t,a_t) - \log\frac{q_c(a_t|x_t)}{\pi(a_t|x_t)} - \log\frac{q_d(x_{t+1}|x_t,a_t)}{p(x_{t+1}|x_t,a_t)}\big].
% \label{eq:ELBO2}
% \end{align}
% \end{small}
% \vspace{-0.15in}

%Note that in (\ref{eq:ELBO2}), $\mathbb E_q[\cdot]=\mathbb E[\cdot|p_0,q_d,q_c]$. In order 

To maximize $\mathcal J(q;\pi)$ w.r.t.~$q$ and $\pi$, we first fix $\pi$ and compute the variational distribution {\bf (E-step)}:

\vspace{-0.15in}
\begin{small}
\begin{align}
q^*=(q_c^*,q_d^*)\in \! \argmax_{q_c\in\Delta_{\mathcal A},q_d\in\Delta_{\mathcal X}} \! \mathbb E\big[\sum_{t=0}^{T-1}\eta\cdot r(\obs_t,\control_t) - \log\frac{q_c(a_t|x_t)}{\pi(a_t|x_t)} - \log\frac{q_d(x_{t+1}|x_t,a_t)}{p(x_{t+1}|x_t,a_t)}\mid p_0,q_d,q_c\big],
\label{eq:E-Step1}
\end{align}
\end{small}
\vspace{-0.15in}

and then optimize $\pi$ given $q^*=(q_c^*,q_d^*)$, i.e.,~$\argmax_\pi\mathcal J(q^*;\pi)$ {\bf (M-step)}. Note that in (\ref{eq:E-Step1}), $q_c^*$ and $q_d^*$ are both functions of $\pi$, but we remove $\pi$ from the notation to keep it lighter.

% \begin{remark}
% Note that if the policy space considered in the M-step, $\Pi$, is the entire simplex $\Delta_{\mathcal A}$, or more generally contains the policy space used for $q_c$ in the E-step, then we can trivially solve the M-step by setting $\pi=q_c^*$. 
% \end{remark}

\begin{remark}
In our formulation (choice of the variational distribution $q$), the M-step is independent of the true dynamics, $p$, and thus, can be implemented offline (using samples generated by the model $q_d$). Moreover, as will be seen in Section~\ref{sec:RL-algo}, we also use the model, $q_d$, in the E-step. As discussed throughout the paper, using simulated samples (from $q_d$) and reducing the need for real samples (from $p$) is an important feature of our proposed model-based formulation and algorithm.
\end{remark}

\begin{remark}
There are similarities between our variational formulation and the one used in the maximum a posteriori policy optimization (MPO) algorithm~\citep{Abdolmaleki18MP}. However, MPO sets its variational dynamics, $q_d$, to the dynamics of the real system, $p$, which results in a model-free algorithm, while our approach is model-based, since we learn $q_d$ and use it to generate samples in both E-step and M-step of our algorithms. 
\end{remark}

In the rest of this section, we study the E-step optimization (\ref{eq:E-Step1}) and propose algorithms to solve it.

%%%%%%%%%%%%%%%%%%%%%%%%%%%%%%%%%%%%%%%%%%%%%%%%%%%%%%%%%%%%%%%%%%%%%%%%%%
%%%%%%%%%%%%%%%%%%%%%%%%%%%%%%%%%%%%%%%%%%%%%%%%%%%%%%%%%%%%%%%%%%%%%%%%%%
%%%%%%%%%%%%%%%%%%%%%%%%%%%%%%%%%%%%%%%%%%%%%%%%%%%%%%%%%%%%%%%%%%%%%%%%%%

\vspace{-0.1in}
\subsection{Properties of the E-step Optimization}
\label{subsec:E-step-properties}
\vspace{-0.1in}

We start by defining two Bellman-like operators related to the E-step optimization (\ref{eq:E-Step1}). For any variational policy $q_c\in\Delta_{\mathcal A}$ and any value function $V:\mathcal X\rightarrow \mathbb R$, such that $V(x)=0,\;\forall x\in{\mathcal X}^0$, we define the {\em $q_c$-induced operator} and the {\em optimal operator} as 

\vspace{-0.15in}
\begin{small}
\begin{align}
\label{eq:qc-Operator}
{\mathcal T}_{q_c}[V](x) &:= \mathbb E_{a\sim q_c(\cdot|x)}\Big[\eta\cdot r(x,a) - \log\frac{q_c(a|x)}{\pi(a|x)} +\! \max_{q_d\in\Delta_{\mathcal X}}\mathbb E_{x'\sim q_d(\cdot|x,a)}\big[V(x') - \log\frac{q_d(x'|x,a)}{p(x'|x,a)}\big]\Big], \\
{\mathcal T}[V](x) &:= \max_{q_c\in\Delta_{\mathcal A}}{\mathcal T}_{q_c}[V](x).
\label{eq:optimal-Operator}
\end{align}
\end{small}
\vspace{-0.15in}

%for all $x\in\mathcal X\backslash{\mathcal X}^0$ and ${\mathcal T}_{q_c}[f](x) := 0,\;\forall x\in{\mathcal X}^0$. 
%and the {\em optimal operator} as ${\mathcal T}[V](x):=\max_{q_c\in\Delta_{\mathcal A}}{\mathcal T}_{q_c}[V](x)$. 
We also define the {\em optimal value function} of the E-step, $V_\pi$, as 

\vspace{-0.15in}
\begin{small}
\begin{equation}
V_\pi(x) := \mathbb E\big[\sum_{t=0}^{T-1}\eta\cdot r(\obs_t,\control_t)-\log\frac{q^*_c(a_t|x_t)}{\pi(a_t|x_t)}-\log\frac{q^*_d(x_{t+1}|x_t,a_t)}{p(x_{t+1}|x_t,a_t)}\mid p_0,q^*_d,q^*_c\big].
\label{eq:optimal-V-E-step}
\end{equation}
\end{small}
\vspace{-0.125in}

For any value function $V$, we define its associated action-value function $Q:\mathcal X\times\mathcal A\rightarrow \mathbb R$ as 

\vspace{-0.125in}
\begin{small}
\begin{equation}
\label{eq:Q-function}
Q(x,a):=\eta\cdot r(\obs,\control)+\log\mathbb E_{x'\sim p(\cdot|x,a)}\big[\exp\big(V(x')\big)\big].
\end{equation}
\end{small}
\vspace{-0.15in}

We now prove (in Appendices~\ref{sec:prelim_tech_results} and~\ref{appendix:lem:fixed_point}) the following lemmas about the properties of these operators, ${\mathcal T}_{q_c}$ and $\mathcal T$, and their relation with the (E-step) optimal value function, $V_\pi$. 

\vspace{0.05in}

\begin{lemma}
\label{lem:optimal-operator1}
We may rewrite the $q_c$-induced and optimal operators defined by (\ref{eq:qc-Operator}) and (\ref{eq:optimal-Operator}) as

\vspace{-0.15in}
\begin{small}
\begin{align}
\label{eq:induced-operator1}
{\mathcal T}_{q_c}[V](x) &= \mathbb E_{a\sim q_c(\cdot|x)}\big[Q(x,a) - \log\frac{q_c(a|x)}{\pi(a|x)}\big], \\
{\mathcal T}[V](x) &= \log \mathbb E_{a\sim\pi(\cdot|x),x'\sim p(\cdot|x,a)}\big[\exp\big(\eta\cdot r(x,a)+V(x')\big)\big].
\label{eq:optimal-operator1}
\end{align}
\end{small}
\end{lemma}

\vspace{0.05in}
\begin{lemma}
\label{lem:qc-induced-optimal-operators}
The $q_c$-induced, ${\mathcal T}_{q_c}$, and optimal, ${\mathcal T}$, operators are monotonic and a contraction. Moreover, the optimal value function $V_\pi$ is the unique fixed-point of $\mathcal T$, i.e.,~${\mathcal T}[V_\pi](x)=V_\pi(x),\;\forall x\in\mathcal X$. 
\end{lemma}
\vspace{-0.05in}

\vspace{0.05in}

From the definition of $Q$-function in (\ref{eq:Q-function}) and Lemma~\ref{lem:qc-induced-optimal-operators}, we prove (in Appendix~\ref{subsec:proof-prop1}) the following proposition for the action-value function associated with the E-step optimal value function $V_\pi$.

\vspace{0.05in}

\begin{proposition}
\label{prop:Q-function}
The E-step optimal value function, $V_\pi$, and its associated action-value function, $Q_\pi$, defined by (\ref{eq:Q-function}), have the following relationship: $V_\pi(x) = \log\mathbb E_{a\sim\pi(\cdot|x)}\big[\exp\big(Q_\pi(x,a)\big)\big], \quad \forall x\in\mathcal X$.
% \vspace{-0.1in}
% \begin{small}
% \begin{equation}
% \label{eq:optimal-V-Q-functions}
% V_\pi(x) = \log\mathbb E_{a\sim\pi(\cdot|x)}\big[\exp\big(Q_\pi(x,a)\big)\big], \quad \forall x\in\mathcal X.
% \end{equation}
% \end{small}
% \vspace{-0.15in}

% This further implies the following Bellman-like recursion for $Q_
% \pi$: \mgh{This might be redundant.}

% \vspace{-0.1in}
% \begin{small}
% \begin{equation}
% \label{eq:optimal-Q-recursion}
% Q_\pi(x,a) = \eta\cdot r(x,a) + \log\mathbb E_{x'\sim p(\cdot|x,a),a'\sim\pi(\cdot|x')}\big[\exp\big(Q_\pi(x',a')\big)\big], \quad \forall x\in\mathcal X,\;\forall a\in\mathcal A.
% \end{equation}
% \end{small}
% \vspace{-0.15in}
%
\end{proposition}

% Alternatively, denote the corresponding state-action value as
% \begin{equation}\label{eq:Q}
% Q(x,a):=\eta\cdot r(\obs,\control)+\log\int_{x'\in\mathcal X}P(x'|x,a)\exp V(x'),\,\forall x\in\mathcal X,
% \end{equation}
% and $Q(x,a)=0$ at any $x\in\mathcal X^0$, $\forall a\in\mathcal A$. When $V=V_\pi$, which is the value function of the optimistic-MDP ELBO problem, using the above results we have that $Q_\pi(x,a)=\eta\cdot r(\obs,\control)+\log\int_{x'\in\mathcal X}P(x'|x,a)\exp V_\pi(x')$, and with the results from Lemma \ref{lem:fixed_point}
% \begin{equation}\label{eq:q_v_equiv}
% \begin{split}
% V_\pi(x)=&\max_{q_c\in\Delta}q_c(a|x)\bigg(\eta\cdot r(\obs,\control)+\log\int_{x'\in\mathcal X}\!\!\!\!P(x'|x,a)\exp V_\pi(x')-\log\frac{q_c(a|x)}{\pi(a|x)}\bigg)\\
% =&\max_{q_c\in\Delta}q_c(a|x)\bigg(Q_\pi(x,a)-\log\frac{q_c(a|x)}{\pi(a|x)}\bigg)=\log\int_{a}\pi(a|x)\exp Q_\pi(x,a),\,\,\forall x\in\mathcal X.
% \end{split}
% \end{equation}
% This further implies that
% \[
% Q_\pi(x,a)=\eta\cdot r(\obs,\control)+\log\int_{x'\in\mathcal X}\!\!\!\!P(x'|x,a)\int_{a'\in\mathcal A}\pi(a'|x')\exp Q_\pi(x',a'),\,\,\forall x\in\mathcal X.
% \]
% at any $x\in\mathcal X$, and $Q(x,a)=0$ at any $x\in\mathcal X^0$, $\forall a\in\mathcal A$.

%%%%%%%%%%%%%%%%%%%%%%%%%%%%%%%%%%%%%%%%%%%%%%%%%%%%%%%%%%%%%%%%%%%%%%%%%%
%%%%%%%%%%%%%%%%%%%%%%%%%%%%%%%%%%%%%%%%%%%%%%%%%%%%%%%%%%%%%%%%%%%%%%%%%%
%%%%%%%%%%%%%%%%%%%%%%%%%%%%%%%%%%%%%%%%%%%%%%%%%%%%%%%%%%%%%%%%%%%%%%%%%%

%\subsection{Posterior Distributions in Closed-form}
%\label{subsec:posteriors-closed-form}

In the rest of this section, we show how to derive a closed-form expression for the variational distribution $q^*=(q^*_c,q^*_d)$. For any value function $V$, we define its corresponding variational dynamics, $q_d^V$, as the solution to the maximization problem in the definition of ${\mathcal T}_{q_c}$ (see Eq.~\ref{eq:qc-Operator}), i.e.,

\vspace{-0.125in}
\begin{small}
\begin{equation}
\label{eq:posterior-dynamics}
q_d^V(\cdot|x,a)\in\argmax_{q_d\in\Delta_{\mathcal X}}\;\mathbb E_{x'\sim q_d(\cdot|x,a)}\big[V(x')-\log\frac{q_d(x'|x,a)}{p(x'|x,a)}\big],
\end{equation}
\end{small}
\vspace{-0.1in}

and its corresponding variational policy, $q_c^Q$, where $Q$ is the action-value function associated with $V$ (Eq.~\ref{eq:Q-function}), as the solution to the maximization problem in the definition of $\mathcal T$ (see Eqs.~\ref{eq:optimal-Operator} and~\ref{eq:induced-operator1}), i.e.,

\vspace{-0.15in}
\begin{small}
\begin{equation}
\label{eq:posterior-policy}
q_c^Q(\cdot|x)\in\argmax_{q_c\in\Delta_{\mathcal A}}\;\mathbb E_{a\sim q_c(\cdot|x)}\big[Q(x,a)-\log\frac{q_c(a|x)}{\pi(a|x)}\big].
\end{equation}
\end{small}
\vspace{-0.15in}

%where $Q$ is the action-value function associated with $V$ (see Eq.~\ref{eq:Q-function}). 
%In the following lemma (proof in Appendix~\ref{app:proofs1}), we derive closed-form expressions for the variational distributions $q_d^V$ and $q_c^Q$.
%, defined by (\ref{eq:posterior-dynamics}) and (\ref{eq:posterior-policy}). 
%\vspace{0.075in}

We now derive closed-form expressions for the variational distributions $q_d^V$ and $q_c^Q$.

\begin{lemma}\label{lem:policies}
The variational dynamics and policy corresponding to a value function $V$ and its associated action-value function $Q$ can be written in closed-form as \hfill {\em (proof in Appendix~\ref{subsec:proof-lemma-variational-distributions})}

\vspace{-0.125in}
\begin{small}
\begin{align}
\label{eq:posterior-dynamics-closed-form}
q_d^V(x'|x,a) &= \frac{p(x'|x,a)\cdot\exp\big(V(x')\big)}{\mathbb E_{x'\sim p(\cdot|x,a)}\big[\exp\big(V(x')\big)\big]} = \frac{p(x'|x,a)\cdot\exp\big(V(x')\big)}{\exp\big(Q(x,a) - \eta\cdot r(x,a)\big)}, \quad\forall x,x'\in\mathcal X,\;\forall a\in\mathcal A, \\
q_c^Q(a|x) &= \frac{\pi(a|x)\cdot\exp\big(Q(x,a)\big)}{\mathbb E_{a\sim\pi(\cdot|x)}\big[\exp\big(Q(x,a)\big)\big]}, \quad\forall x\in\mathcal X,\;\forall a\in\mathcal A.
\label{eq:posterior-policy-closed-form}
\end{align}
\end{small}
\vspace{-0.125in}

\end{lemma}

From (\ref{eq:posterior-dynamics-closed-form}) and (\ref{eq:posterior-policy-closed-form}), the variational dynamics, $q_d^V$, and policy, $q_c^Q$, can be seen as an {\em exponential twisting} of the dynamics $p$ and policy $\pi$ with weights $V$ and $Q$, respectively. In the special case $V=V_\pi$ (the E-step optimal value function), these distributions can be written in closed-form as

\vspace{-0.125in}
\begin{small}
\begin{equation}
\label{eq:optimal-posteriors-closed-form}
q^*_d(x'|x,a) = \frac{p(x'|x,a)\cdot\exp\big(V_\pi(x')\big)}{\exp\big(Q_\pi(x,a) - \eta\cdot r(x,a)\big)}, \qquad\quad q^*_c(a|x) = \frac{\pi(a|x)\cdot\exp\big(Q_\pi(x,a)\big)}{\exp\big(V_\pi(x)\big)},
\end{equation}
\end{small}
\vspace{-0.125in}

where the denominator of $q^*_c$ is obtained by applying Proposition~\ref{prop:Q-function} to replace $Q_\pi$ with $V_\pi$.

\vspace{-0.1in}
\subsection{Policy and Value Iteration Algorithms for the E-step}
\label{subsec:PI-VI-Algos}
\vspace{-0.1in}

Using the results of Section~\ref{subsec:E-step-properties}, we now propose {\em model-based} and {\em model-free} dynamic programming (DP) style algorithms, i.e.,~policy iteration (PI) and value iteration (VI), for solving the E-step problem (\ref{eq:E-Step1}). The model-based algorithms compute the variational dynamics, $q_d$, at each iteration, while the model-free ones compute $q_d$ only at the end (upon convergence). Having access to $q_d$ at each iteration has the advantage that we may generate samples from the model, $q_d$, when we implement the sample-based version (RL version) of these DP algorithms in Section~\ref{sec:RL-algo}.

In the {\bf model-based PI} algorithm, at each iteration $k$, given the current variational policy $q_c^{(k)}$, we 

{\em Policy Evaluation:} Compute the $q_c^{(k)}$-induced value function $V_{q_c^{(k)}}$ (the fixed-point of the operator $\mathcal T_{q_c^{(k)}}$) by iteratively applying $\mathcal T_{q_c^{(k)}}$ from Eq.~\ref{eq:qc-Operator}, i.e.,~$V_{q_c^{(k)}}(x)=\lim_{n\rightarrow\infty}\mathcal T_{q_c^{(k)}}^n[V](x),\;\forall x\in\mathcal X$, where the variational model $q_d$ in (\ref{eq:qc-Operator}) is computed using Eq.~\ref{eq:posterior-dynamics-closed-form} with $V=V^{(n)}$. We then compute the corresponding action-value function, $Q_{q_c^{(k)}}$ using Eq.~\ref{eq:Q-function}.

{\em Policy Improvement:} Update the variational distribution $q_c^{(k+1)}$ using Eq.~\ref{eq:posterior-policy-closed-form} with $Q=Q_{q_c^{(k)}}$.\footnote{When the number of actions is large, the denominator of (\ref{eq:posterior-policy-closed-form}) cannot be computed efficiently. In this case, we replace (\ref{eq:posterior-policy-closed-form}) in the policy improvement step of our PI algorithms with $q_c^{(k+1)}= \argmin_{q_c}\text{KL}(q_c||q_c^Q)$, where $Q=Q_{q_c^{(k)}}$. We also prove the convergence of our PI algorithms with this update in Appendix~\ref{app:proof-lemma5}.}

%{\em Policy Improvement:} Update the variational distribution $q_c^{(k+1)}$, either using Eq.~\ref{eq:posterior-policy-closed-form} or using $q_c^{(k+1)}(\cdot|x)\leftarrow\min_{q_c\in\Delta_{\mathcal A}}\text{KL}( q_c(\cdot|\obs)||q_c^Q(\cdot|x))$, with $Q=Q_{q_c^{(k)}}$, at any $x\in\mathcal X$.

Upon convergence, i.e.,~$q_c^{(\infty)}=q_c^*$, we compute $q_d^*$ from Eq.~\ref{eq:posterior-dynamics-closed-form} and return $q=(q_c^*,q_d^*)$.

The {\bf model-free PI} algorithm is exactly the same, except in its policy evaluation step, the $q_c^{(k)}$-induced operator, $\mathcal T_{q_c^{(k)}}$, is applied using Eq.~\ref{eq:induced-operator1} (without the variational dynamics $q_d$). In this case, the variational dynamics, $q_d$, is computed only upon convergence, $q^*_d$, using Eq.~\ref{eq:posterior-dynamics-closed-form}.

\begin{lemma}
\label{lem:policy_iteration}
To solutions returned by the model-based and model-free PI algorithms converge to their optimal values, $q_c^*$ and $q_d^*$, defined by (\ref{eq:E-Step1}), i.e.,~$q_c^{(\infty)}=q_c^*$ and $q_d^{(\infty)}=q_d^*$. \hfill {\em (proof in Appendix~\ref{app:proof-lemma5})}
\end{lemma}

We can similarly derive {\bf model-based} and {\bf model-free value iteration (VI)} algorithms for the E-step. These algorithms start from an arbitrary value function, $V$, and iteratively apply the optimal operator, $\mathcal T$, from Eqs.~\ref{eq:qc-Operator} and~\ref{eq:optimal-Operator} (model-based) and Eq.~\ref{eq:optimal-operator1} (model-free) until convergence, i.e.,~$V_\pi(x) = \lim_{n\rightarrow\infty}\mathcal T^n[V](x),\;\forall x\in\mathcal X$. Given $V_\pi$, these algorithms first compute $Q_\pi$ from Proposition~\ref{prop:Q-function}, and then compute ($q_c^*$, $q_d^*$) using Eq.~\ref{eq:optimal-posteriors-closed-form}. From the properties of the optimal operator $\mathcal T$ in Lemma~\ref{lem:qc-induced-optimal-operators}, it is easy to see that both model-based and model-free VI algorithms converge to $q_c^*$ and $q_d^*$. 

In this paper, we focus on the PI algorithms, in particular the model-based one, and leave the VI algorithms for future work. In the next section, we show how the PI algorithms can be implemented and combined with a routine for solving the M-step, when the true MDP model, $p$, is unknown (the RL setting) and the state and action spaces are large that require using function approximation.

\vspace{-0.1in}
\section{Variational Model-based Policy Optimization Algorithm}
\label{sec:RL-algo}
\vspace{-0.1in}

In this section, we propose a RL algorithm, called variational model-based policy optimization (VMBPO). VMBPO is a EM-style algorithm based on the variational formulation proposed in Section~\ref{sec:elbo}. The E-step of VMBPO is the sample-based implementation of the model-based PI algorithm, described in Section~\ref{subsec:PI-VI-Algos}. %The goal of the E-step is to find the solution $q^*=(q_c^*,q_d^*)$ to the optimization problem (\ref{eq:E-Step1}), given a policy $\pi$. Then in the M-step, the algorithm optimizes for $\pi$, given the varitional distribution $q^*$ computed in the E-step. Algorithm~\ref{alg:VMBPO} contains the pseudo-code of VMBPO. 
We describe the E-step and M-step of VMBPO in details in Sections~\ref{subsec:VMBPO-E-step} and~\ref{subsec:VMBPO-M-step}, and report its pseudo-code in Algorithm~\ref{alg:VMBPO} in Appendix~\ref{sec:VMBPO-algo}. VMBPO uses $8$ neural networks to represent: policy $\pi$, variational dynamics $q_d$, variational policy $q_c$, log-likelihood ratio $\nu=\log(q_d/p)$, value function $V$, action-value function $Q$, target value function $V'$, and target action-value function $Q'$, with parameters $\theta_\pi$, $\theta_d$, $\theta_c$, $\theta_\nu$, $\theta_v$, $\theta_q$, $\theta'_v$, and $\theta'_q$, respectively. 

%%%%%%%%%%%%%%%%%%%%%%%%%%%%%%%%%%%%%%%%%%%%%%%%%%%%%%%%%%%%%%%%%%%%%%%%%%
%%%%%%%%%%%%%%%%%%%%%%%%%%%%%%%%%%%%%%%%%%%%%%%%%%%%%%%%%%%%%%%%%%%%%%%%%%
%%%%%%%%%%%%%%%%%%%%%%%%%%%%%%%%%%%%%%%%%%%%%%%%%%%%%%%%%%%%%%%%%%%%%%%%%%

\vspace{-0.1in}
\subsection{The E-step of VMBPO}
\label{subsec:VMBPO-E-step}
\vspace{-0.1in}

At the beginning of the E-step, we generate a number of samples $(x,a,r,x')$ from the current baseline policy $\pi$, i.e.,~$a\sim\pi(\cdot|x)$ and $r=r(x,a)$, and add them to the buffer $\mathcal D$. The E-step consists of four updates: %{\bf 1)} computing the target action-value function $Q'$, 
{\bf 1)} computing the variational dynamics $q_d$, {\bf 2)} estimating the log-likelihood ratio $\log(q_d/p)$, {\bf 3)} computing the $q_c$-induced value, $V_{q_c}$, and action-value, $Q_{q_c}$, functions (critic update), and finally {\bf 4)} computing the variational policy new $q_c$ (actor update). We describe the details of each step below.   

% {\bf Step 1. (Computing $Q'$)} $\;$ The goal of this step is to set the target action-value function $Q'$ to an approximation of the target value function $V'$ using Eq.~\ref{eq:Q-function}. Thus, we update $Q'$ parameter $\theta'_q$ by taking several steps in the direction of the gradient of the following loss function: 

% \vspace{-0.15in}
% \begin{small}
% \begin{equation}
% \label{eq:target-Q-update}
% \theta'_q = \argmin_\theta \sum_{(x,a,r,x')\sim\mathcal D}\big(\exp(Q'(x,a;\theta) - \eta \cdot r) - \exp(V'(x';\theta'_v))\big)^2,
% \end{equation}
% \end{small}
% \vspace{-0.15in}

% where the loss in (\ref{eq:target-Q-update}) is the result of talking exponential from both sides of (\ref{eq:Q-function}) and the data tuples $(x,a,r,x')$ have been randomly sampled from the buffer $\mathcal D$.

{\bf Step 1. (Computing $q_d$)} $\;$ We find $q_d$ as the solution to the optimization problem (\ref{eq:posterior-dynamics}) for $V$ equal to the target value network $V'$. Since the $q^V_d$ in (\ref{eq:posterior-dynamics-closed-form}) is the solution of (\ref{eq:posterior-dynamics}), we compute $q_d$ by minimizing $\text{KL}(q_d^{V'} || q_d)$, which results in the following {\em forward} KL loss, for all $x\in\mathcal X$ and $a\in\mathcal A$: 

\vspace{-0.15in}
\begin{small}
\begin{align}
\label{eq:qd-update1}
\theta_d &= \argmin_\theta \; \text{KL}\big(p(\cdot|x,a) \cdot \exp(\eta\cdot r(x,a) + V'(\cdot;\theta'_v) - Q'(x,a;\theta'_q)) \; || \; q_d(\cdot|x,a;\theta)\big) \\
&\stackrel{\text{(a)}}{=} \argmax_\theta \; \mathbb E_{x'\sim p(\cdot|x,a)}\big[\exp(\eta\cdot r(x,a) + V'(x';\theta'_v) - Q'(x,a;\theta'_q)) \cdot \log(q_d(\cdot|x,a;\theta))\big], 
\label{eq:qd-update2}
\end{align}
\end{small}
\vspace{-0.15in}

where {\bf (a)} is by removing the $\theta$-independent terms from (\ref{eq:qd-update1}). We update $\theta_d$ by taking several steps in the direction of the gradient of a sample average of the loss function (\ref{eq:qd-update2}), i.e.,

\vspace{-0.125in}
\begin{small}
\begin{equation}
\label{eq:qd-update3}
\theta_d = \argmax_\theta \; \sum_{(x,a,r,x')\sim\mathcal D}\exp(\eta\cdot r + V'(x';\theta'_v) - Q'(x,a;\theta'_q)) \cdot \log\big(q_d(x'|x,a;\theta)\big), 
\end{equation}
\end{small}
\vspace{-0.125in}

where $(x,a,r,x')$ are randomly sampled from $\mathcal D$. The intuition here is to focus on learning the dynamics model in regions of the state-action space that has higher temporal difference --- regions with higher anticipated future return.
Note that we can also obtain $\theta_d$ by optimizing the {\em reverse} KL direction in (\ref{eq:qd-update1}), but since it results in a more involved update, we do not report it here. 

{\bf Step 2. (Computing $\log (q_d/ p)$)} $\;$ Using the duality of f-divergence~\citep{nguyen2008estimating} w.r.t.~the {\em reverse} KL-divergence, the log-likelihood ratio $\log(q_d(\cdot|x,a;\theta_d)/p(\cdot|x,a))$ is a solution to 

\vspace{-0.15in}
\begin{small}
\begin{equation}
\label{eq:nu_update}
\log\big(\frac{q_d(x'|x,a;\theta_d)}{p(x'|x,a)}\big) = \argmax_{\nu:\mathcal X\times\mathcal A\times\mathcal X\rightarrow\mathbb R}\mathbb E_{x'\sim q_d(\cdot|x,a;\theta_d)}[\nu(x'|x,a)] - \mathbb E_{x'\sim p(\cdot|x,a)}[\exp(\nu(x'|x,a))],
\end{equation}
\end{small}
\vspace{-0.15in}

for all $x,x'\in\mathcal X$ and $a\in\mathcal A$. Note that the optimizer of (\ref{eq:nu_update}) is unique almost surely (at $(x,a,x')$ with $\mathbb P(x'|x,a)>0$), because $q_d$ is absolutely continuous w.r.t.~$p$ (see the definition of $q_d$ in Eq.~\ref{eq:posterior-dynamics-closed-form}) and the objective function of (\ref{eq:nu_update}) is strictly concave. The optimization problem (\ref{eq:nu_update}) allows us to compute $\nu(\cdot|\cdot;\theta_\nu)$ as an approximation to the log-likelihood ratio $\log(q_d(\cdot;\theta_d)/p)$. We update $\theta_\nu$ by taking several steps in the direction of the gradient of a sample average of (\ref{eq:nu_update}), i.e.,

\vspace{-0.125in}
\begin{small}
\begin{equation}
\label{eq:nu_update-2}
\theta_\nu = \argmax_\theta \sum_{(x,a,x')\sim\mathcal E}\nu(x'|x,a;\theta) - \sum_{(x,a,x')\sim\mathcal D}\exp(\nu(x'|x,a;\theta)),
\end{equation}
\end{small}
\vspace{-0.15in}

where $\mathcal E$ is the set of samples for which $x'$ is drawn from the variational dynamics, i.e.,~$x'\sim q_d(\cdot|x,a)$. 
Here we first sample $(x,a,x')$ randomly from $\mathcal D$ and use them in the second sum. Then, for all $(x,a)$ that have been sampled, we generate $x'$ from $q_d$ and use the resulting samples in the first sum. 

{\bf Step 3. (critic update)} $\;$ To compute $V_{q_c}$ (fixed-point of $\mathcal T_{q_c}$) and its action-value $Q_{q_c}$, we first rewrite (\ref{eq:qc-Operator}) with the maximizer $q_d$ from Step~1 and the log-likelihood ratio $\log(q_d/p)$ from Step~2:

\vspace{-0.15in}
\begin{small}
\begin{equation}
\label{eq:qc-Operator-temp1}
\mathcal T_{q_c}[V](x) = \mathbb E_{a\sim q_c(\cdot|x)}\big[\eta\cdot r(x,a) - \log\frac{q_c(a|x)}{\pi(a|x)} + \mathbb E_{x'\sim q_d(\cdot|x,a;\theta_d)}[V'(x';\theta'_v) - \nu(x'|x,a;\theta_\nu)]\big]. % - \text{KL}(q_d(\cdot|x,a;\theta_d) \; || \; p(\cdot|x,a))\big].
\end{equation}
\end{small}
\vspace{-0.15in}

Since $\mathcal T_{q_c}$ can be written as both (\ref{eq:induced-operator1}) and (\ref{eq:qc-Operator-temp1}), we compute the $q_c$-induced $Q$-function by setting the RHS of these equations equal to each other, i.e.,~for all $x\in\mathcal X$ and $a\sim q_c(\cdot|x;\theta_c)$,

\vspace{-0.15in}
\begin{small}
\begin{equation}
\label{eq:critic-Q-1}
Q(x,a;\theta_q) = \eta\cdot r(x,a) + \mathbb E_{x'\sim q_d(\cdot|x,a;\theta_d)}[V'(x';\theta'_v) - \nu(x'|x,a;\theta_\nu)].
\end{equation}
\end{small}
\vspace{-0.15in}

Since the expectation in (\ref{eq:critic-Q-1}) is w.r.t.~the variational dynamics (model) $q_d$, we can estimate $Q_{q_c}$ only with samples generated from the model. We do this by taking several steps in the direction of the gradient of a sample average of the square-loss obtained by setting the two sides of (\ref{eq:critic-Q-1}) equal, i.e.,

\vspace{-0.125in}
\begin{small}
\begin{equation}
\label{eq:critic-Q-2}
\theta_q = \argmin_\theta \sum_{(x,a,r,x')\sim\mathcal E} \big(Q(x,a;\theta) - \eta \cdot r - V'(x';\theta'_v) + \nu(x'|x,a;\theta_\nu)\big)^2.
\end{equation}
\end{small}
\vspace{-0.125in}

Note that in (\ref{eq:critic-Q-1}), the actions are generated by $q_c$. Thus, in (\ref{eq:critic-Q-2}), we first randomly sample $x$, then sample $a$ from $q_c(\cdot|x;\theta_c)$, and finally draw $x'$ from $q_d(\cdot|x,a;\theta_d)$. If the reward function is known (chosen by the designer of the system), then it is used to generate the reward signals $r=r(x,a)$ in (\ref{eq:critic-Q-2}), otherwise, a reward model has to be learned.

After estimating $Q_{q_c}$, we approximate $V_{q_c}$, the fixed-point of $\mathcal T_{q_c}$, using $\mathcal T_{q_c}$ definition in (\ref{eq:induced-operator1}) as 
$\mathcal T_{q_c}[V](x) \approx V(x) \approx \mathbb E_{a\sim q_c(\cdot|x)}\big[Q(x,a;\theta_q) - \log\frac{q_c(a|x;\theta_c)}{\pi(a|x;\theta_\pi)}\big]$.
This results in updating $V_{q_c}$ by taking several steps in the direction of the gradient of a sample average of the square-loss obtained by setting the two sides of the above equation to be equal, i.e.,

\vspace{-0.125in}
\begin{small}
\begin{equation}
\label{eq:critic-V-2}
\theta_v = \argmin_\theta \sum_{(x,a)\sim\mathcal E} \big(V(x;\theta) - Q(x,a;\theta_q) + \log\frac{q_c(a|x;\theta_c)}{\pi(a|x;\theta_\pi)}\big)^2,
\end{equation}
\end{small}
\vspace{-0.125in}

where $x$ is randomly sampled and $a\sim q_c(\cdot|x;\theta_c)$ (without sampling from the true environment).

{\bf Step 4. (actor update)} $\;$ We update the variational policy $q_c$ (policy improvement) by solving the optimization problem (\ref{eq:posterior-policy}) for the $Q$ estimated by the critic in Step~3. Since the $q_c$ that optimizes (\ref{eq:posterior-policy}) can be written as (\ref{eq:posterior-policy-closed-form}), we update it by minimizing $\text{KL}(q_c||q_c^Q)$. This results in the following {\em reverse} KL loss, for all $x\in\mathcal X$:  
$
\theta_c = \argmin_\theta \; \text{KL}\big(q_c(\cdot|x;\theta) || \frac{\pi(\cdot|x;\theta_\pi)\cdot\exp(Q(x,\cdot,;\theta_q))}{Z(x)}\big) = \argmin_\theta \; \mathbb E_{a\sim q_c}\big[\log(\frac{q_c(a|x;\theta)}{\pi(a|x;\theta_\pi)}) - Q(x,a,;\theta_q)\big]
$.
If we reparameterize $q_c$ using a transformation $a=f(x,\epsilon;\theta_c)$, where $\epsilon$ is a Gaussian noise, we can update $\theta_c$ by taking several steps in the direction of the gradient of a sample average of the above loss, i.e.,

\vspace{-0.125in}
\begin{small}
\begin{equation}
\label{eq:actor-update-2}
\theta_c = \argmin_\theta \sum_{(x,\epsilon)} \log\big(q_c(f(x,\epsilon;\theta)|x)) - Q(x,a,;\theta_q) - \log(\pi(a|x;\theta_\pi)\big).
\end{equation}
\end{small}
\vspace{-0.125in}

We can also compute $q_c$ as the closed-form solution to (\ref{eq:posterior-policy-closed-form}), as described in~\citet{Abdolmaleki18MP}. They refer to this as non-parametric representation of the variational distribution.

\vspace{-0.1in}
\subsection{The M-step of VMBPO}
\label{subsec:VMBPO-M-step}
\vspace{-0.1in}

As described in Section~\ref{sec:elbo}, the goal of the M-step is to improve the baseline policy $\pi$, given the variational model $q^*=(q_c^*,q_d^*)$ learned in the E-step, by solving the following optimization problem

\vspace{-0.15in}
\begin{small}
\begin{align}
\pi \leftarrow \argmax_{\pi\in\Pi} \; \mathcal J(q^*;\pi) := \mathbb E_{q^*}\big[\sum_{t=0}^{T-1} \eta\cdot r(x_t,a_t) - \log\frac{q^*_c(a_t|x_t)}{\pi(a_t|x_t)} - \log\frac{q^*_d(x_{t+1}|x_t,a_t)}{p(x_{t+1}|x_t,a_t)}\big].
\label{eq:M-step-Opt1}
\end{align}
\end{small}
\vspace{-0.15in}

A nice feature of (\ref{eq:M-step-Opt1}) is that it can be solved using only the variational model $q^*$, without the need for samples from the true environment $p$. However, it is easy to see that if the policy space considered in the M-step, $\Pi$, contains the one used for $q_c$ in the E-step, then we can trivially solve the M-step by setting $\pi=q_c^*$. Although this is an option, it is more efficient in practice to solve a regularized version of (\ref{eq:M-step-Opt1}). A practical way to regularize (\ref{eq:M-step-Opt1}) is to make sure that the new baseline policy $\pi$ remains close to the old one, which results in the following optimization problem  

\vspace{-0.15in}
\begin{small}
\begin{align}
\theta_\pi \leftarrow \argmax_{\theta} \; \mathbb E_{q^*}\big[\sum_{t=0}^{T-1} \log(\pi(a_t|x_t;\theta)) - \lambda\cdot\text{KL}\big(\pi(\cdot|x_t;\theta_\pi)||\pi(\cdot|x_t;\theta)\big)\big].
\label{eq:M-step-Opt2}
\end{align}
\end{small}
\vspace{-0.15in}

This is equivalent to the weighted MAP formulation used in the M-step of MPO~\citep{Abdolmaleki18MP}. In MPO, they define a prior over the parameter $\theta$ and add it as $\log P(\theta)$ to the objective function of (\ref{eq:M-step-Opt1}). Then, they set the prior $P(\theta)$ to a specific Gaussian and obtain an optimization problem similar to (\ref{eq:M-step-Opt2}) (see Section~3.3 in~\citealt{Abdolmaleki18MP}). However, since in their variational model $q_d=p$ (their approach is model-free), they need real samples to solve their optimization problem, while we can solve (\ref{eq:M-step-Opt2}) only by simulated samples (our approach is model-based).

\vspace{-0.1in}
\section{Experiments}
\label{sec:experiments}
\vspace{-0.1in}

To illustrate the effectiveness of VMBPO, we (i) compare it with several state-of-the-art RL methods on multiple domains, and (ii) assess the trade-off between sample efficiency via ablation analysis. 

\vspace{-0.125in}
\paragraph{Comparison with Baseline RL Algorithms}

We compare VMBPO with five baseline methods, MPO \citep{Abdolmaleki18MP}, SAC \citep{Haarnoja18SA} ---two popular model-free deep RL algorithms---and STEVE \citep{buckman2018sample}, PETS \citep{chua2018unsupervised}, and MBPO \citep{Janner19WT} ---three recent model-based RL algorithms. We also compare with the (E-step) model-free variant of VMBPO, which is known as VMBPO-MFE (see Appendix \ref{sec:model-free-E-step-Algo} for details). We evaluate the algorithms on one classical control benchmark (Pendulum) and five MuJoCo benchmarks (Hopper, Walker2D, HalfCheetah, Reacher, Reacher7DoF). The neural network architectures (for the the dynamics model, value functions, and policies) of VMBPO are similar to that of MBPO. Details on network architectures and hyperparameters are described in Appendix~\ref{appendix:experimental_details}. Since we parameterize $q_c$ in the E-step of VMBPO, according to Section \ref{subsec:VMBPO-M-step}, in the M-step we simply set $\pi=q_c^*$.
For the more difficult environments (Walker2D, HalfCheetah), the number of training steps is set to $400,000$, while for the medium one (Hopper) and for the simpler ones (Pendulum, Reacher, Reacher7DOF), it is set to $150,000$ and $50,000$ respectively. 
Policy performance is evaluated every $1000$ training iterations.
Each measurement is an average return over $5$ episodes, each generated with a separate random seed.
To smooth learning curves, data points are averaged over a sliding window of size $3$.

\begin{table}[th!]
\begin{adjustwidth}{-.5in}{-.5in}
\centering
\scalebox{0.7}{
\begin{tabular}{|l|c|c|c|c|c|c|c|}
\hline
\textbf{Env.} & \textbf{VMBPO} & \textbf{MBPO} & \textbf{STEVE} & \textbf{PETS} & \textbf{VMBPO-MFE} & \textbf{SAC} & \textbf{MPO} \\ [0.5ex]
\hline
\hline
Pendulum & -125.8 $\pm$ 73.7 & -126.0 $\pm$ 78.4 & -6385.3 $\pm$ 799.7 & -183.5 $\pm$ 1773.9 & -125.7 $\pm$ 130.1 & \textbf{-124.7} $\pm$ 199.0 & -131.9 $\pm$ 315.9  \\
\hline
Hopper & \textbf{2695.2} $\pm$ 902.1 & 2202.8 $\pm$ 938.3 & 279.0 $\pm$ 237.1 & 94.5 $\pm$ 114.2 & 1368.7 $\pm$ 184.1 & 2020.8 $\pm$ 954.1 & 1509.7 $\pm$ 756.0 \\
\hline
Walker2D & \textbf{3592.2} $\pm$ 1068.0 & 3574.9 $\pm$ 815.9 & 336.3 $\pm$ 196.3 & 93.5 $\pm$ 134.1 & 3334.5 $\pm$ 122.8 & 3026.4 $\pm$ 888.9 & 2889.4 $\pm$ 712.7 \\
\hline
HalfCheetah & 10712.4 $\pm$ 1266.9 & 10652.1 $\pm$ 899.4 & 482.9 $\pm$ 596.9 & \textbf{13272.6} $\pm$ 4926.4 & 4647.3 $\pm$ 505.8 & 9080.3 $\pm$ 1625.1 & 4969.2 $\pm$ 623.7 \\
\hline
Reacher & \textbf{-11.4} $\pm$ 27.0 & -12.6 $\pm$ 25.9 & -141.8 $\pm$ 355.7 & --- & -55.5 $\pm$ 39.0 & -23.9 $\pm$ 23.8 & -75.9 $\pm$ 336.7 \\
\hline
Reacher7DoF & \textbf{-13.8} $\pm$ 20.5 & -15.1 $\pm$ 98.8 & --- & -45.6 $\pm$ 36.1 & -33.5 $\pm$ 49.6 & -27.4 $\pm$ 112.0 & -38.4 $\pm$ 53.8 \\
\hline
\end{tabular}
}
\end{adjustwidth}
\caption{The mean $\pm$ standard deviation of final returns with the best hyper-parameter configuration. VMBPO significantly outperforms other baselines. VMBPO-MFE can improve over MPO but is quite unstable.}
\label{table:exp1_best_mean}
\vspace{-0.05in}
\end{table}

\begin{figure}[th!]
\centering
\includegraphics[width=0.88\textwidth]{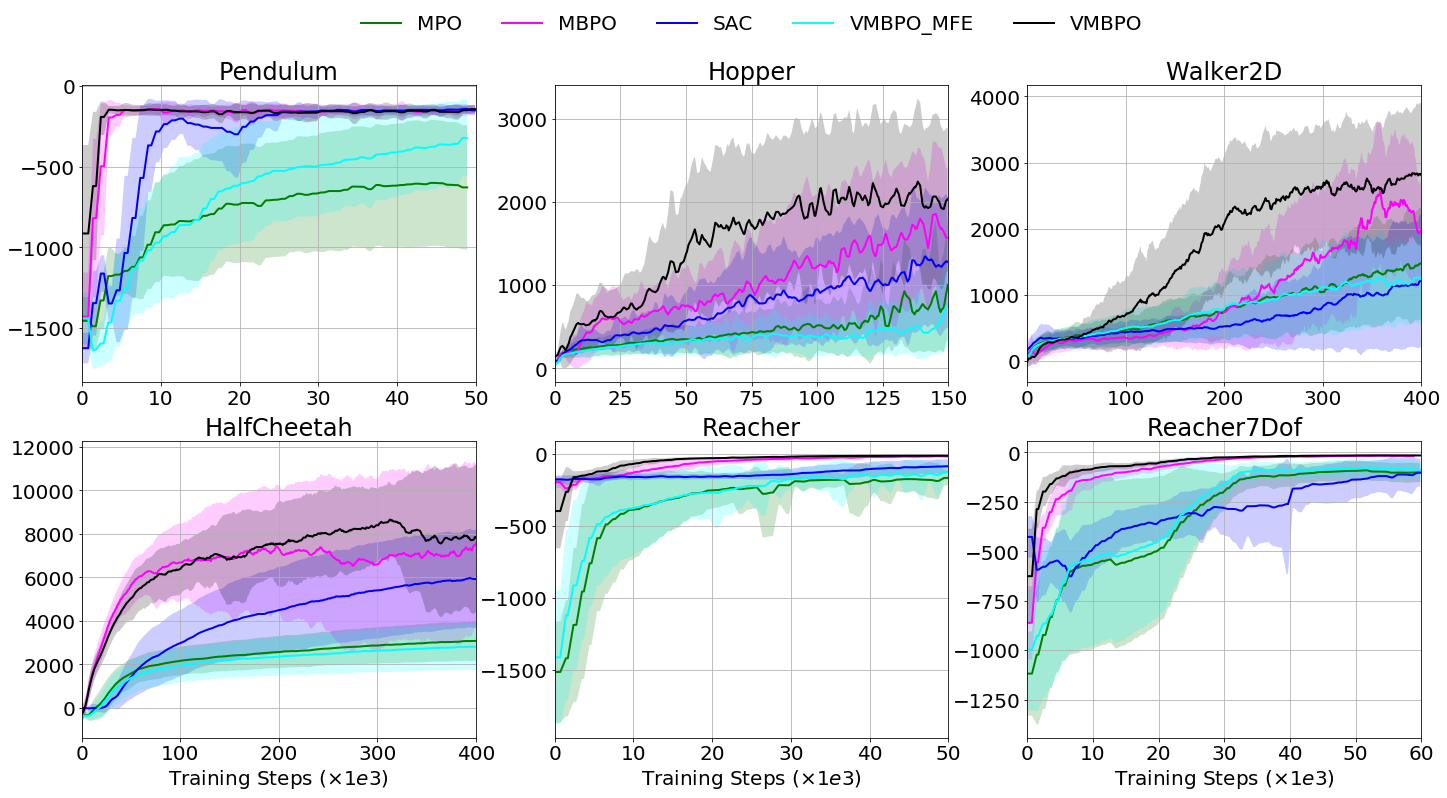}
\vspace{-0.1in}
\caption{Mean cumulative reward over all hyper-parameter and random-seed configurations. We do not include PETS and STEVE because the hyper-parameters are adopted from their papers.}
\label{fig:exp1_all_mean}
\vspace{-0.05in}
\end{figure}

Table~\ref{table:exp1_best_mean} and Figure~\ref{fig:exp1_best_mean} (Appendix \ref{appendix:additional_exp}) show the average return of VMBPO, VMBPO-MFE, and the baselines under the best hyperparameter configurations. VMBPO outperforms the baseline algorithms in most of the benchmarks, with significantly faster convergence speed and a higher reward. This verifies our conjecture about VMBPO:
(i) Utilizing synthetic data from the learned dynamics model generally improves data-efficiency of RL; (ii) Extra improvement in VMBPO is attributed to the fact that model is learned from the universal RL objective function.   
On the other hand,  VMBPO-MFE outperforms MPO in 4 out of 6 domains.
However, in some cases the learning may experience certain degradation issues (which lead to poor performance). This is due to the instability caused by sample variance amplification in critic learning with exponential-TD minimization (see Eq.~\ref{eq:vmbpo_mfe_critic_update_q} in Section \ref{sec:vmbpo_mfe_critic_update}). To alleviate this issue one may introduce a temperature term $\tau>0$ to the exponential-TD update \citep{borkar2002q}. However, tuning this hyper-parameter can be quite non-trivial.\footnote{The variance is further amplified with a large $\tau$, but the critic learning is hampered by a small $\tau$.}
Table~\ref{table:exp1_all_mean} (Appendix \ref{appendix:additional_exp}) and Figure~\ref{fig:exp1_all_mean} show the summary statistics averaged over all hyper-parameter/random-seed configurations and illustrate the sensitivity to hyperparameters of each method. VMBPO is more robust (with best performance on all the tasks) than other baselines.
This corroborates with the hypothesis that MBRL generally is more robust to hyperparameters than its model-free counterparts. 

\begin{figure}[th!]
\centering
\includegraphics[width=0.29\textwidth]{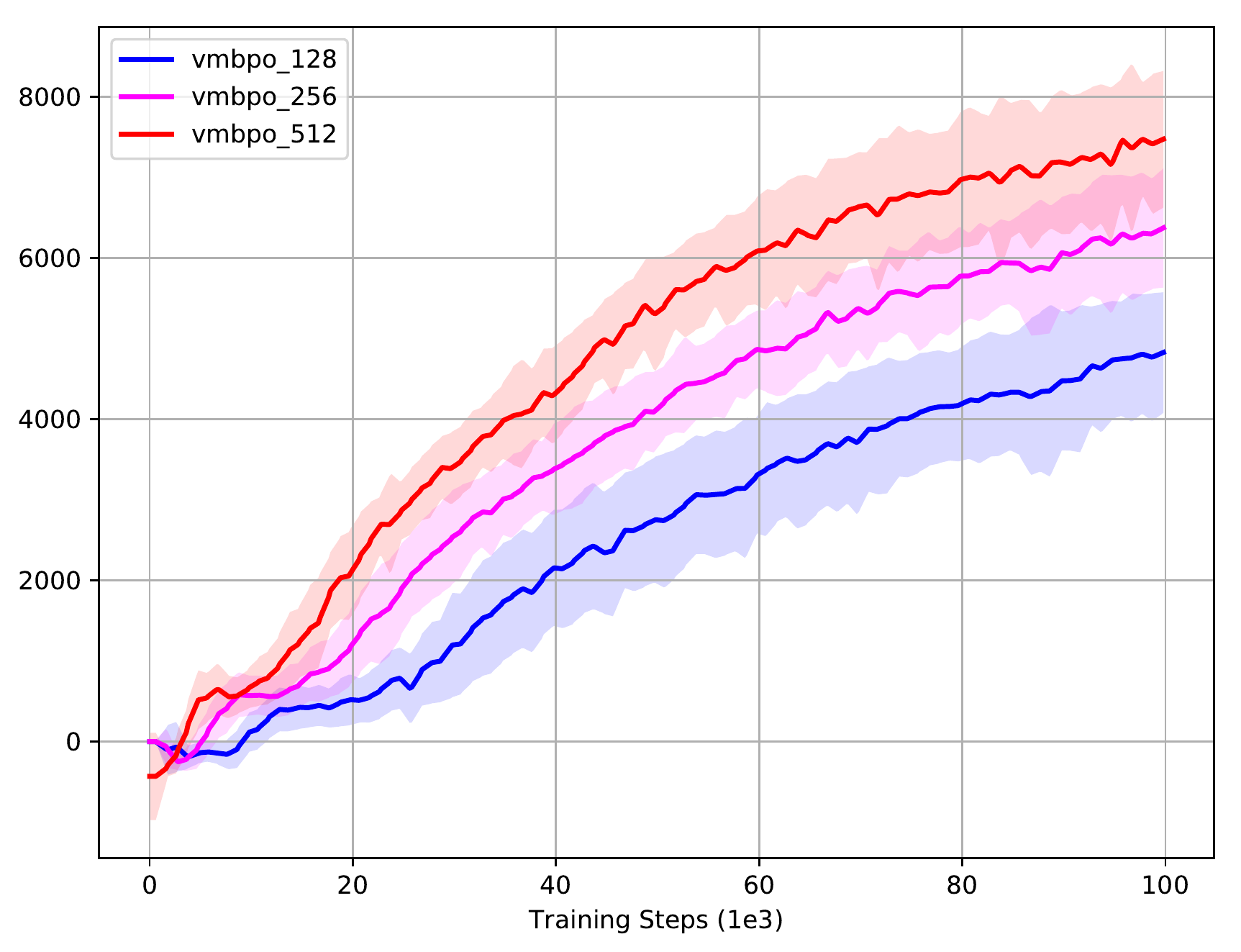}\includegraphics[width=0.29\textwidth]{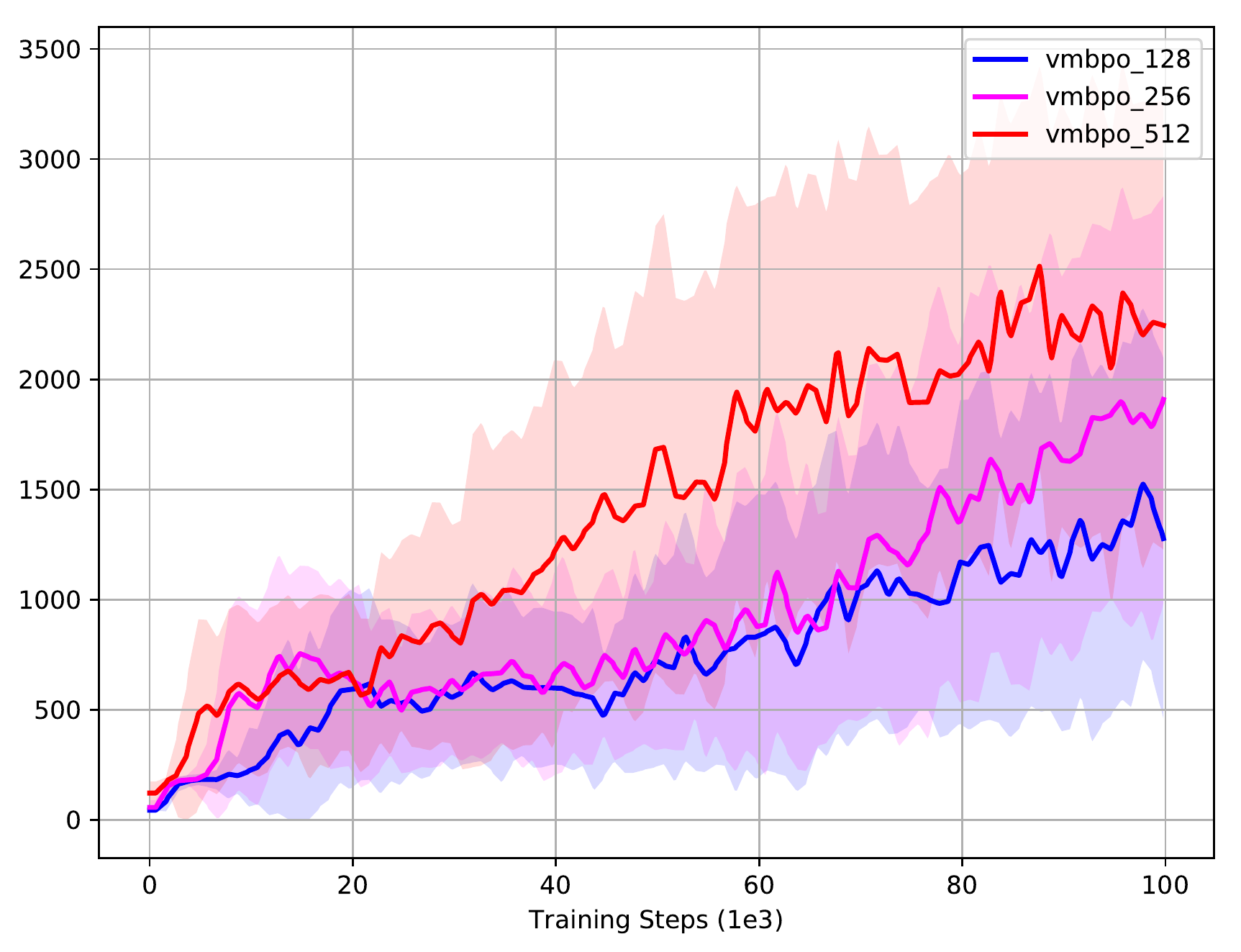}\includegraphics[width=0.29\textwidth]{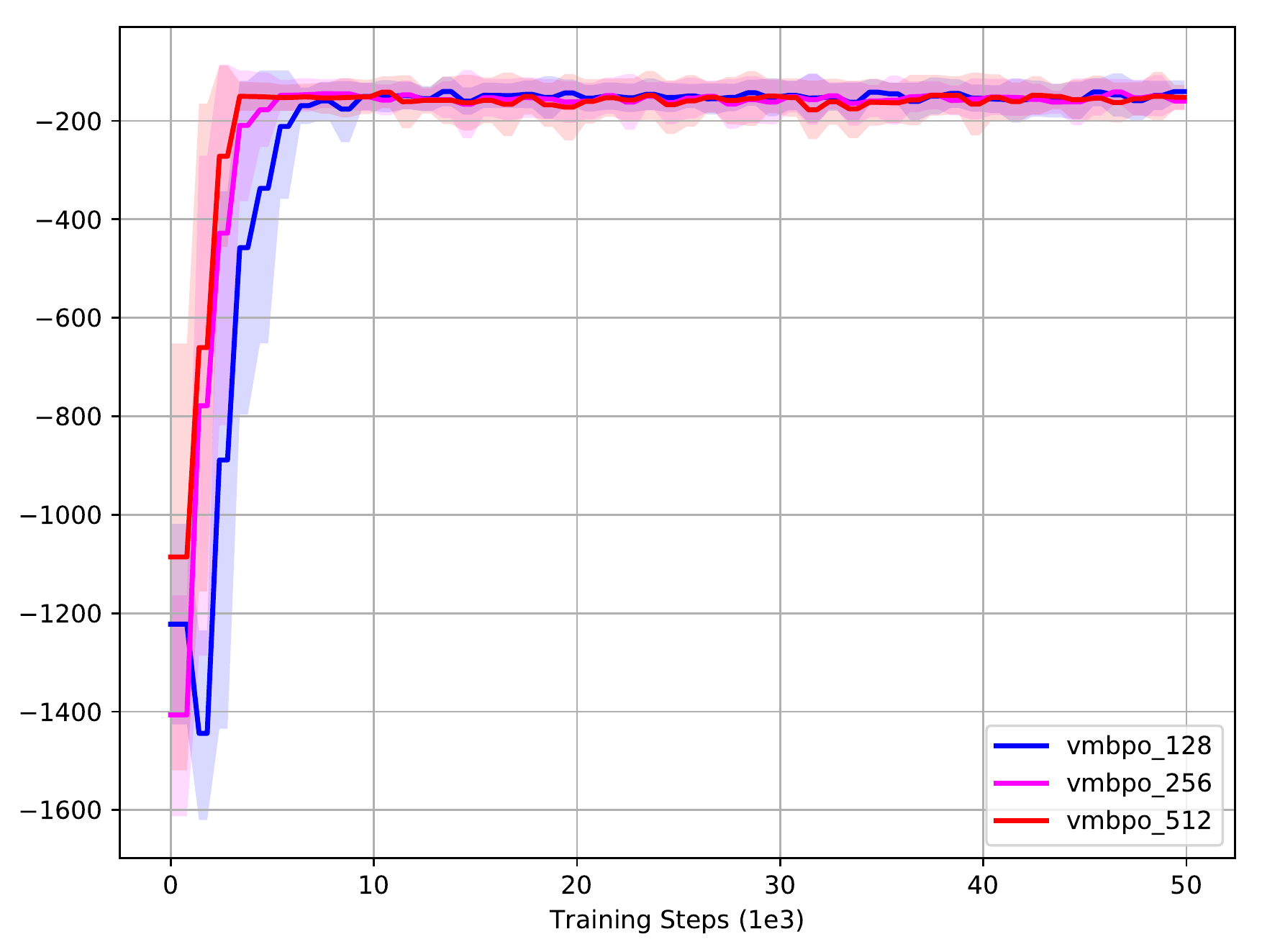}
\vspace{-0.1in}
\caption{Ablation analysis on RL data efficiency w.r.t. synthetic data generated by model $q_d$.}
\label{fig:exp2_best_mean}
\vspace{-0.05in}
\end{figure}

\vspace{-0.12in}
\paragraph{Ablation Analysis}

We now study the effects of data-efficiency of VMBPO w.r.t. the data samples generated from the dynamics model $q_d$. For simplicity, we only experiment with 3 standard benchmarks (Pendulum, Hopper, HalfCheetah) and with fewer learning steps ($50,000$, $100,000$, $100,000$). At each step, we update the actor and critic using $\{128, 256, 512\}$ synthetic samples.
Figure~\ref{fig:exp2_best_mean} shows the statistics averaged over all hyper-parameter/random-seed configurations and illustrates how synthetic data can help with policy learning. The results show that increasing the amount of synthetic data generally improves the policy convergence rate. In the early phase when the dynamics model is inaccurate, sampling data from it may slow down learning, while in the later phase with an improved model adding more synthetic data leads to a more significant performance boost.

%%%%%%%%%%%%%%%%%%%%%%%%%%%%%%%%%%%%%%%%%%%%%%%%%%%%%%%%%%%%%%%%%%%%%%%%%%
%%%%%%%%%%%%%%%%%%%%%%%%%%%%%%%%%%%%%%%%%%%%%%%%%%%%%%%%%%%%%%%%%%%%%%%%%%
%%%%%%%%%%%%%%%%%%%%%%%%%%%%%%%%%%%%%%%%%%%%%%%%%%%%%%%%%%%%%%%%%%%%%%%%%%
%%%%%%%%%%%%%%%%%%%%%%%%%%%%%%%%%%%%%%%%%%%%%%%%%%%%%%%%%%%%%%%%%%%%%%%%%%
%%%%%%%%%%%%%%%%%%%%%%%%%%%%%%%%%%%%%%%%%%%%%%%%%%%%%%%%%%%%%%%%%%%%%%%%%%

\vspace{-0.1in}
\section{Conclusion}
\label{sec:conclu}
\vspace{-0.1in}

We formulated the problem of jointly learning and improving model and policy in RL as a variational lower-bound of a log-likelihood, and proposed EM-type algorithms to solve it. Our algorithms, called variational model-based policy optimization (VMBPO) use model-based policy iteration for solving the E-step. We compared our (E-step) model-based and model-free algorithms with each other, and with a number of state-of-the-art model-based (e.g.,~MBPO) and model-free (e.g.,~MPO) RL algorithms, and showed its sample efficiency and performance. 

We briefly discussed VMBPO style algorithms in which the E-step is solved by model-based policy iteration methods. However, full implementation of these algorithms and studying their relationship with the existing methods requires more work that we leave for future. Another future directions are: {\bf 1)} finding more efficient implementation for VMBPO, and {\bf 2)} using VMBPO style algorithms in solving control problems from high-dimensional observations, by learning a low-dimensional latent space and a latent space dynamics, and perform control there. This class of algorithms is referred to as learning controllable embedding~\citep{Watter15EC,Levine20PC}.

%%%%%%%%%%%%%%%%%%%%%%%%%%%%%%%%%%%%%%%%%%%%%%%%%%%%%%%%%%%%%%%%%%%%%%%%%%
%%%%%%%%%%%%%%%%%%%%%%%%%%%%%%%%%%%%%%%%%%%%%%%%%%%%%%%%%%%%%%%%%%%%%%%%%%
%%%%%%%%%%%%%%%%%%%%%%%%%%%%%%%%%%%%%%%%%%%%%%%%%%%%%%%%%%%%%%%%%%%%%%%%%%
%%%%%%%%%%%%%%%%%%%%%%%%%%%%%%%%%%%%%%%%%%%%%%%%%%%%%%%%%%%%%%%%%%%%%%%%%%
%%%%%%%%%%%%%%%%%%%%%%%%%%%%%%%%%%%%%%%%%%%%%%%%%%%%%%%%%%%%%%%%%%%%%%%%%%

\newpage
\section*{Broader Impact}

This work proposes methods to jointly learn and improve model and policy in reinforcement learning. We see learning control-aware models as a promising direction to address an important challenge in model-based reinforcement learning: creating a balance between the bias in simulated data and the ease of data generation (sample efficiency).  

%%%%%%%%%%%%%%%%%%%%%%%%%%%%%%%%%%%%%%%%%%%%%%%%%%%%%%%%%%%%%%%%%%%%%%%%%%
%%%%%%%%%%%%%%%%%%%%%%%%%%%%%%%%%%%%%%%%%%%%%%%%%%%%%%%%%%%%%%%%%%%%%%%%%%
%%%%%%%%%%%%%%%%%%%%%%%%%%%%%%%%%%%%%%%%%%%%%%%%%%%%%%%%%%%%%%%%%%%%%%%%%%
%%%%%%%%%%%%%%%%%%%%%%%%%%%%%%%%%%%%%%%%%%%%%%%%%%%%%%%%%%%%%%%%%%%%%%%%%%
%%%%%%%%%%%%%%%%%%%%%%%%%%%%%%%%%%%%%%%%%%%%%%%%%%%%%%%%%%%%%%%%%%%%%%%%%%

\bibliographystyle{plainnat}
\bibliography{V-MBPO-Reference}

%%%%%%%%%%%%%%%%%%%%%%%%%%%%%%%%%%%%%%%%%%%%%%%%%%%%%%%%%%%%%%%%%%%%%%%%%%
%%%%%%%%%%%%%%%%%%%%%%%%%%%%%%%%%%%%%%%%%%%%%%%%%%%%%%%%%%%%%%%%%%%%%%%%%%
%%%%%%%%%%%%%%%%%%%%%%%%%%%%%%%%%%%%%%%%%%%%%%%%%%%%%%%%%%%%%%%%%%%%%%%%%%
%%%%%%%%%%%%%%%%%%%%%%%%%%%%%%%%%%%%%%%%%%%%%%%%%%%%%%%%%%%%%%%%%%%%%%%%%%
%%%%%%%%%%%%%%%%%%%%%%%%%%%%%%%%%%%%%%%%%%%%%%%%%%%%%%%%%%%%%%%%%%%%%%%%%%

\newpage
\appendix

%%%%%%%%%%%%%%%%%%%%%%%%%%%%%%%%%%%%%%%%%%%%%%%%%%%%%%%%%%%%%%%%%%%%%%%%%%
%%%%%%%%%%%%%%%%%%%%%%%%%%%%%%%%%%%%%%%%%%%%%%%%%%%%%%%%%%%%%%%%%%%%%%%%%%
%%%%%%%%%%%%%%%%%%%%%%%%%%%%%%%%%%%%%%%%%%%%%%%%%%%%%%%%%%%%%%%%%%%%%%%%%%
%%%%%%%%%%%%%%%%%%%%%%%%%%%%%%%%%%%%%%%%%%%%%%%%%%%%%%%%%%%%%%%%%%%%%%%%%%
%%%%%%%%%%%%%%%%%%%%%%%%%%%%%%%%%%%%%%%%%%%%%%%%%%%%%%%%%%%%%%%%%%%%%%%%%%

\section{Proofs of Section~\ref{sec:elbo}}
\label{app:proofs1}

%%%%%%%%%%%%%%%%%%%%%%%%%%%%%%%%%%%%%%%%%%%%%%%%%%%%%%%%%%%%%%%%%%%%%%%%%%
%%%%%%%%%%%%%%%%%%%%%%%%%%%%%%%%%%%%%%%%%%%%%%%%%%%%%%%%%%%%%%%%%%%%%%%%%%
%%%%%%%%%%%%%%%%%%%%%%%%%%%%%%%%%%%%%%%%%%%%%%%%%%%%%%%%%%%%%%%%%%%%%%%%%%

\subsection{Proof of Lemma~\ref{lem:optimal-operator1}}
\label{sec:prelim_tech_results}

Before proving Lemma~\ref{lem:optimal-operator1}, we first state and prove the following results.

\begin{lemma}
\label{lem:reg_max}
For any state $x\in\mathcal X$, action-value function $Q$, and policy $\pi$, we have
\begin{equation}
\label{eq:reg_max1}
\max_{q_c\in\Delta_{\mathcal A}}\mathbb E_{a\sim q_c(\cdot|x)}\Big[Q(x,a)-\log\frac{q_c(a|x)}{\pi(a|x)} \Big] = \log\mathbb E_{a\sim\pi(\cdot|x)}\big[\exp\big(Q(x,a)\big)\big].
\end{equation}
Analogously, for any state-action pair $(x,a)\in\mathcal X\times\mathcal A$, value function $V$, and transition kernel $p(\cdot|x,a)$, we have
\begin{equation}
\label{eq:reg_max2}
\max_{q_d\in\Delta_{\mathcal X}}\mathbb E_{x'\sim q_d(\cdot|x,a)}\Big[V(x') - \log\frac{q_d(x'|x,a)}{p(x'|x,a)}\Big] = \log\mathbb E_{x'\sim p(\cdot|x,a)}\big[\exp\big(V(x')\big)\big].
\end{equation}
\end{lemma}
\begin{proof}
We only prove (\ref{eq:reg_max1}) here, since the proof of (\ref{eq:reg_max2}) follows similar arguments. The proof of (\ref{eq:reg_max1}) comes from the following sequence of equalities:

\vspace{-0.15in}
\begin{small}
\begin{align*}
\max_{q_c\in\Delta_{\mathcal A}}\mathbb E_{a\sim q_c(\cdot|x)}\Big[Q(x,a)&-\log\frac{q_c(a|x)}{\pi(a|x)} \Big] = \max_{q_c\in\Delta_{\mathcal A}}\mathbb E_{a\sim q_c(\cdot|x)}[(Q(x,a)+\log\pi(a|x))-\log q_c(a|x)] \\ 
&\stackrel{\text{(a)}}{=} \log\int_a\exp (Q(x,a)+\log\pi(a|x)) = \log\mathbb E_{a\sim\pi(\cdot|x)}\big[\exp\big(Q(x,a)\big)\big].
\end{align*}
\end{small}
\vspace{-0.15in}

{\bf (a)} This follows from Lemma~4 in~\citet{Nachum17BG}.
\end{proof}

We now turn to the proof of our main lemma.
\begin{proof}[Proof of Lemma~\ref{lem:optimal-operator1}] From (\ref{eq:qc-Operator}), for any state $x\in\mathcal X$, we may write the $q_c$-induced operator as

\vspace{-0.15in}
\begin{small}
\begin{align*}
{\mathcal T}_{q_c}[V](x) &= \mathbb E_{a\sim q_c(\cdot|x)}\Big[\eta\cdot r(x,a) - \log\frac{q_c(a|x)}{\pi(a|x)} + \max_{q_d\in\Delta_{\mathcal X}}\mathbb E_{x'\sim q_d(\cdot|x,a)}\big[V(x') - \log\frac{q_d(x'|x,a)}{p(x'|x,a)}\big]\Big] \\    
&\stackrel{\text{(a)}}{=} \mathbb E_{a\sim q_c(\cdot|x)}\Big[\eta\cdot r(x,a) - \log\frac{q_c(a|x)}{\pi(a|x)} + \log\mathbb E_{x'\sim p(\cdot|x,a)}\big[\exp(V(x'))\big]\Big] \\
&\stackrel{\text{(b)}}{=} \mathbb E_{a\sim q_c(\cdot|x)}\Big[Q(x,a) - \log\frac{q_c(a|x)}{\pi(a|x)}\Big].
\end{align*}
\end{small}
\vspace{-0.15in}

{\bf (a)} From Lemma~\ref{lem:reg_max}.

{\bf (b)} From the definition of the $Q$-function in (\ref{eq:Q-function}).

This concludes the proof of (\ref{eq:induced-operator1}), the first statement of Lemma~\ref{lem:optimal-operator1}. Now to prove the second statement (Eq.~\ref{eq:optimal-operator1}), we may write

\vspace{-0.15in}
\begin{small}
\begin{align*}
\mathcal T[V](x) &= \max_{q_c\in\Delta_{\mathcal A}}\mathbb E_{a\sim q_c(\cdot|x)}\Big[\eta\cdot r(x,a)-\log\frac{q_c(a|x)}{\pi(a|x)} +\max_{q_d\in\Delta_{\mathcal X}}\mathbb E_{x'\sim q_d(\cdot|x,a)}\big[V(x') - \log\frac{q_d(x'|x,a)}{p(x'|x,a)}\big]\Big] \\
&\stackrel{\text{(a)}}{=} \max_{q_c\in\Delta_{\mathcal A}}\mathbb E_{a\sim q_c(\cdot|x)}\Big[\eta\cdot r(x,a)-\log\frac{q_c(a|x)}{\pi(a|x)} + \log\mathbb E_{x'\sim p(\cdot|x,a)}\big[\exp(V(x'))\big]\Big] \\
%&= \max_{q_c\in\Delta_{\mathcal A}}\int_{a}q_c(a|x) \Big(\eta\cdot r(x,a) + \log\mathbb E_{x'\sim p(\cdot|x,a)}\big[\exp\big(V(x')\big)\big] - \log\frac{q_c(a|x)}{\pi(a|x)}\Big) \nonumber \\
&\stackrel{\text{(b)}}{=} \log\mathbb E_{a\sim\pi(\cdot|x)}\Big[\exp\Big(\eta\cdot r(x,a) + \log\mathbb E_{x'\sim p(\cdot|x,a)}\big[\exp\big(V(x')\big)\big]\Big)\Big] \\
&= \log\mathbb E_{a\sim\pi(\cdot|x)}\Big[\exp\big(\eta\cdot r(x,a)\big)\mathbb E_{x'\sim p(\cdot|x,a)}\big[\exp\big(V(x')\big)\big]\Big] \\ 
&= \log\mathbb E_{a\sim\pi(\cdot|x),x'\sim p(\cdot|x,a)}\big[\exp\big(\eta\cdot r(x,a)+V(x')\big)\big].
\end{align*}
\end{small}
\vspace{-0.15in}

{\bf (a)} and {\bf (b)} both come from Lemma~\ref{lem:reg_max}. 

This concludes the proof of (\ref{eq:optimal-operator1}), the second statement of Lemma~\ref{lem:optimal-operator1}.
\end{proof}

%%%%%%%%%%%%%%%%%%%%%%%%%%%%%%%%%%%%%%%%%%%%%%%%%%%%%%%%%%%%%%%%%%%%%%%%%%
%%%%%%%%%%%%%%%%%%%%%%%%%%%%%%%%%%%%%%%%%%%%%%%%%%%%%%%%%%%%%%%%%%%%%%%%%%
%%%%%%%%%%%%%%%%%%%%%%%%%%%%%%%%%%%%%%%%%%%%%%%%%%%%%%%%%%%%%%%%%%%%%%%%%%

%\newpage

%\subsection{Proofs of Lemma~\ref{lem:qc-induced-operator} and Lemma~\ref{lem:optimal-operators}}
\subsection{Proof of Lemma~\ref{lem:qc-induced-optimal-operators}}
\label{appendix:lem:fixed_point}

We only prove the properties of the optimal operator, $\mathcal T$, here. The proof for the $q_c$-induced operator, $\mathcal T_{q_c}$, follows similar arguments. 

\textbf{\noindent{1. Monotonicity:}} 
For any functions $V,\,W:\mathcal X\rightarrow\mathbb R$, such that $V(x)\leq W(x),\;\forall x\in\mathcal X$, we have $\mathcal T[V](x)\leq \mathcal T[W](x),\;\forall x\in\mathcal X$.

\begin{proof}
For the case of $x\in\mathcal X^0$ (the set of terminal states), the property trivially holds. For the case of $x\in\mathcal X$, from the definition of the optimal operator in (\ref{eq:optimal-operator1}), it is easy to see that for all $x\in\mathcal X$, we have

\vspace{-0.15in}
\begin{small}
\begin{equation*}
\eta \cdot r(x,a)+\log\int_{x'}p(x'|x,a)\exp\big(V(x')\big)\leq
\eta \cdot r(x,a)+\log\int_{x'}p(x'|x,a)\exp\big(W(x')\big),
\end{equation*}
\end{small}
\vspace{-0.15in}

which means $\mathcal T[V](x) = \log\mathbb E_{a\sim\pi(\cdot|x),x'\sim p(\cdot|x,a)}\big[\exp\big(\eta\cdot r(x,a) + V(x')\big)\big] \leq \mathcal T[W](x) = \log\mathbb E_{a\sim\pi(\cdot|x),x'\sim p(\cdot|x,a)}\big[\exp\big(\eta\cdot r(x,a) + W(x')\big)\big],\;\forall x\in\mathcal X$. This completes the proof of monotonicity. 
\end{proof}

\textbf{\noindent{2. Contraction:}} There exists a vector with positive components, i.e., $\rho:\mathcal X\rightarrow\mathbb R_{\geq 0}$, and a discounting factor $0<\gamma<1$ such that 

\vspace{-0.15in}
\begin{small}
\begin{equation*}
\|\mathcal T[V]-\mathcal T[W]\|_\rho\leq \gamma\|V-W\|_\rho,
\end{equation*}
\end{small}
\vspace{-0.15in}

where the weighted norm is defined as $\|V\|_\rho=\max_{x\in\mathcal X} \frac{V(x)}{\rho(x)}$.
\begin{proof}
For the case of $x\in\mathcal X^0$ (the set of terminal states), the property trivially holds because the contraction maps to zero. For the case of $x\in\mathcal X$, following the construction in Proposition 3.3.1 in~\citet{bertsekas1995dynamic}, consider a risk-sensitive entropy-regularized stochastic shortest path problem (via dynamic exponential risk formulation from \cite{borkar2002q}), where the reward are all equal to $1/\eta$. Based on similar arguments as in Proposition 3.3.1, there exists a fixed point value function $\hat V$, such that

\vspace{-0.15in}
\begin{small}
\begin{equation*}
\hat V(x)=1 + \max_{q_c\in\Delta_{\mathcal A}} \int_{a}q_c(a|x)\left(\log\mathbb E_{x'\sim p(\cdot|x,a)}\big[\exp\big(\hat V(x')\big)\big]-\log \frac{q_c(a|x)}{\pi(a|x)}\right).
\end{equation*}
\end{small}
\vspace{-0.15in}

Using the results from Lemma \ref{lem:optimal-operator1}, the above statement further implies that:

\vspace{-0.15in}
\begin{small}
\begin{equation*}
\hat V(x)
= 1 + \log\mathbb E_{a\sim\pi(\cdot|x),x'\sim p(\cdot|x,a)}\big[\exp\big(\hat V(x')\big)\big].
\end{equation*}
\end{small}
\vspace{-0.15in}

Notice that $\hat V(x)\geq 1$ for all $x\in\mathcal X$. By defining $\rho(x)=\hat V(x)$, and by constructing $\gamma=\max_{x\in\mathcal X} (\rho(x)-1)/\rho(x)$, one immediately has $0<\gamma<1$, and

\vspace{-0.15in}
\begin{small}
\begin{align*}
&\max_{q_c\in\Delta_{\mathcal A}}\mathbb E_{a\sim q_c(\cdot|x)}\Big[\max_{q_d\in\Delta_{\mathcal X}}\int_{x'}q_d(x'|x,a)\big(V(x') - \log\frac{q_d(x'|x,a)}{p(x'|x,a)}\big)-\log\frac{q_c(a|x)}{\pi(a|x)}\Big]\\
=&\log\mathbb E_{a\sim\pi(\cdot|x),x'\sim p(\cdot|x,a)}\big[\exp\big(\hat V(x')\big)\big]= \rho(x) - 1\leq \gamma\rho(x).
\end{align*}
\end{small}
\vspace{-0.15in}

Then by following the same lines of analysis as in Proposition 1.5.2 of \cite{bertsekas1995dynamic}, one can show that $\mathcal T$ is a contraction operator.
\end{proof}

\noindent{\textbf{3. Unique Fixed-point Solution}} The optimal value function $V_\pi$ is its unique fixed-point, i.e.,~${\mathcal T}[V_\pi](x)=V_\pi(x),\;\forall x\in\mathcal X$. 
\begin{proof}
Let $V_{\pi}(x)$ be the optimal value function of the E-step problem in Eq.~\ref{eq:E-Step1}, and let $V^*$ be a fixed point solution: $V(x)=\mathcal T[V](x)$, for any $x\in\mathcal X$.
For the case when $x\in\mathcal X^0$, the following result trivially holds: $V_{\pi}(x)=\mathcal T[V_{\pi}](x)=V^*(x)=0$. Below, we show the equality for the case of $x_0\in\mathcal X$.

First, we want to show that $V_{\pi}(x_0)\leq V^*(x_0)$. Consider the greedy policy $\overline q^*_c$ constructed from the Bellman operator $\arg\max_{q_c\in\Delta}\mathcal T_{q_c}[V^*](x)$. Recall that $V^*(x)$ is a fixed point solution to $V(x)=\mathcal T[V](x)$, for any $x\in\mathcal X$. 
Then for any bounded initial value function $V_0$, the contraction property of Bellman operator $\mathcal T_{\overline q^*_c}$ implies that 

\vspace{-0.15in}
\begin{small}
\begin{align*}
&V^*(x)=\lim_{n\rightarrow\infty}\mathcal T^n_{\overline q^*_c}[V_0](x)\\
=&\lim_{n\rightarrow\infty}\max_{q_d\in\Delta_{\mathcal X}}\mathbb E\big[\sum_{t=0}^{n-1} \eta\cdot r(\obs_t,\control_t)-\text{KL}(\overline q^*_c||\pi)(\obs_t)-\text{KL}(q_d||p)(\obs_t,\control_t)\mid q_d,\overline q^*_c,P_0\big],
\end{align*}
\end{small}
\vspace{-0.15in}

for which the transient assumption of stopping MDPs further implies that

\vspace{-0.15in}
\begin{small}
\begin{equation*}
V^*(x)=\max_{q_d\in\Delta_{\mathcal X}}\mathbb E\big[\sum_{t=0}^{\mathrm T^*-1} \eta\cdot r(\obs_t,\control_t)-\text{KL}(\overline q^*_c||\pi)(\obs_t)-\text{KL}(q_d||p)(\obs_t,\control_t)\mid q_d,\overline q^*_c,P_0\big].
\end{equation*}
\end{small}
\vspace{-0.15in}

Since $\overline q^*_c$ is a feasible solution to the E-step problem, this further implies that $V_{\pi}(x_0)\leq V^*(x_0)$.

Second, we want to show that $V_{\pi}(x_0)\geq V^*(x_0)$. Consider the optimal policy $q^*_c$ of the E-step problem. Note that $V^*$ is a fixed point solution to equation: $V^*(x)=\mathcal T[V^*](x)$, for any $x\in\mathcal X$. Immediately the above result yields the following inequality:

\vspace{-0.15in}
\begin{small}
\begin{equation*}
V^*(x)=\mathcal T_{\overline q^*_c}[V^*](x)\leq \mathcal T_{q^*_c}[V^*](x),\,\,\forall x\in\mathcal X,
\end{equation*}
\end{small}
\vspace{-0.15in}

the first equality holds because $\overline q^*_c(\cdot|x)$ is the minimizer of the optimization problem in $\mathcal T[V^*](x)$, $x\in\mathcal X$.
By recursively applying  Bellman operator $\mathcal T_{q_c^*}$, one has the following result:

\vspace{-0.15in}
\begin{small}
\begin{align*}
&V^*(x)\leq \lim_{n\rightarrow\infty}\mathcal T^n_{q^*_c}[V^*](x)\\
=&\max_{q_d\in\Delta_{\mathcal X}}\mathbb E\big[\sum_{t=0}^{\mathrm T^*-1} \eta\cdot r(\obs_t,\control_t)-\text{KL}( q^*_c||\pi)(\obs_t)-\text{KL}(q_d||p)(\obs_t,\control_t)\mid q_d, q^*_c,x_0=x\big]=V_{\pi}(x),\,\,\forall x\in\mathcal X.
\end{align*}
\end{small}
\vspace{-0.15in}

Combining the above analysis, we prove the claim of $V_{\pi}(x_0)=V^*(x_0)$, and the greedy policy of the fixed-point equation, i.e., $\overline q^*_c$, is an optimal policy to the E-step problem.
\end{proof}

%%%%%%%%%%%%%%%%%%%%%%%%%%%%%%%%%%%%%%%%%%%%%%%%%%%%%%%%%%%%%%%%%%%%%%%%%%
%%%%%%%%%%%%%%%%%%%%%%%%%%%%%%%%%%%%%%%%%%%%%%%%%%%%%%%%%%%%%%%%%%%%%%%%%%
%%%%%%%%%%%%%%%%%%%%%%%%%%%%%%%%%%%%%%%%%%%%%%%%%%%%%%%%%%%%%%%%%%%%%%%%%%

%\newpage

\subsection{Proof of Proposition~\ref{prop:Q-function}}
\label{subsec:proof-prop1}

\begin{proof}
The proof follows by combining the definition of the $Q$-function (\ref{eq:Q-function}) and Lemma~\ref{lem:qc-induced-optimal-operators} that indicates $V_\pi$ is the unique fixed-point of the optimal operator $\mathcal T$. Therefore, for any $x\in\mathcal X$, we can write 

\vspace{-0.15in}
\begin{small}
\begin{align*}
V_\pi(x) &= \mathcal T[V_\pi](x) \stackrel{\text{(a)}}{=} \log\mathbb E_{a\sim\pi(\cdot|x)}\Big[\exp(\eta\cdot r(x,a)) \cdot \mathbb E_{x'\sim p(\cdot|x,a)}\big[\exp(V(x'))\big]\Big] \\ 
&\stackrel{\text{(b)}}{=} \log\mathbb E_{a\sim\pi(\cdot|x)}\big[\exp(Q(x,a))\big].
\end{align*}
\end{small}
\vspace{-0.15in}

{\bf (a)} This is from (\ref{eq:optimal-operator1}), the second statement of Lemma~\ref{lem:optimal-operator1}.

{\bf (b)} If we apply exponential to both sides of (\ref{eq:Q-function}), we see that what is inside the bracket is equal to $\exp(Q(x,a))$. 

This concludes the proof. 
%This concludes the proof of (\ref{eq:optimal-V-Q-functions}), the first statement of Proposition~\ref{prop:Q-function}. 
%
%The proof of (\ref{eq:optimal-Q-recursion}), the second statement of Proposition~\ref{prop:Q-function}, follows by replacing $V$ with $V_\pi$ in (\ref{eq:Q-function}), and then replacing $V_\pi$ from (\ref{eq:optimal-V-Q-functions}) that was proved above. 
\end{proof}

%%%%%%%%%%%%%%%%%%%%%%%%%%%%%%%%%%%%%%%%%%%%%%%%%%%%%%%%%%%%%%%%%%%%%%%%%%
%%%%%%%%%%%%%%%%%%%%%%%%%%%%%%%%%%%%%%%%%%%%%%%%%%%%%%%%%%%%%%%%%%%%%%%%%%
%%%%%%%%%%%%%%%%%%%%%%%%%%%%%%%%%%%%%%%%%%%%%%%%%%%%%%%%%%%%%%%%%%%%%%%%%%

\subsection{Proof of Lemma~\ref{lem:policies}}
\label{subsec:proof-lemma-variational-distributions}

\begin{proof}
Since the variational policy, $q_c^Q$ is the solution to the optimization problem (\ref{eq:posterior-policy}), following Corollary~6 in~\citet{Nachum17BG}, we may write that 

\vspace{-0.15in}
\begin{small}
\begin{equation*}
q_c^Q(a|x) = \frac{\exp\big(Q(x,a)+\log\pi(a|x)\big)}{\int_a\exp\big(Q(x,a)+\log\pi(a|x)\big)} = \frac{\pi(a|x)\cdot\exp\big(Q(x,a)\big)}{\mathbb E_{a\sim\pi(\cdot|x)}\big[\exp\big(Q(x,a)\big)\big]}\;.
\end{equation*}
\end{small}
\vspace{-0.15in}

This proves (\ref{eq:posterior-policy-closed-form}), the second statement of Lemma~\ref{lem:policies}. To prove (\ref{eq:posterior-dynamics-closed-form}), the first statement of Lemma~\ref{lem:policies}, we use the fact that the variational dynamics, $q_d^V$ is the solution to the optimization problem (\ref{eq:posterior-dynamics}), and thus, following Corollary~6 in~\citet{Nachum17BG}, we may write that 

\vspace{-0.15in}
\begin{small}
\begin{equation*}
q_d^V(x'|x,a) = \frac{\exp\big(V(x')+\log p(x'|x,a)\big)}{\int_a\exp\big(V(x')+\log p(x'|x,a)\big)} = \frac{p(x'|x,a)\cdot\exp\big(V(x')\big)}{\mathbb E_{x'\sim p(\cdot|x,a)}\big[\exp\big(V(x')\big)\big]}\;.
\end{equation*}
\end{small}
\vspace{-0.15in}

This completes the proof. The second equality in (\ref{eq:posterior-dynamics-closed-form}) is straightforward, because by taking exponential from both sides of (\ref{eq:Q-function}), we have

\vspace{-0.15in}
\begin{small}
\begin{equation*}
\mathbb E_{x'\sim p(\cdot|x,a)}\big[\exp\big(V(x')\big)\big] = \exp\big(Q(x,a) - \eta\cdot r(x,a)\big).
\end{equation*}
\end{small}
\end{proof}

%%%%%%%%%%%%%%%%%%%%%%%%%%%%%%%%%%%%%%%%%%%%%%%%%%%%%%%%%%%%%%%%%%%%%%%%%%
%%%%%%%%%%%%%%%%%%%%%%%%%%%%%%%%%%%%%%%%%%%%%%%%%%%%%%%%%%%%%%%%%%%%%%%%%%
%%%%%%%%%%%%%%%%%%%%%%%%%%%%%%%%%%%%%%%%%%%%%%%%%%%%%%%%%%%%%%%%%%%%%%%%%%

%\newpage
\subsection{Proof of Lemma~\ref{lem:policy_iteration}}
\label{app:proof-lemma5}

In Lemma~\ref{lem:qc-induced-optimal-operators}, we proved that the $q_c$-induced operator, $\mathcal T_{q_c}$, is monotonic and contraction. Therefore, it is clear that for any $q_c$, starting from an arbitrary value function $V$ and iteratively applying $\mathcal T_{q_c}$, we will converge to the fixed-point of this operator, i.e.,~$V_{q_c}=\mathcal T_{q_c}V_{q_c}$. This proves that the policy evaluation step at each iteration $k$ takes $q_c^{(k)}$, as input and returns the $q_c^{(k)}$-induced value function $V_{q_c^{(k)}}$. What needs to be proved is the policy improvement step to show that $V_{q_c^{(k+1)}}(x)\geq V_{q_c^{(k)}}(x),\;\forall x\in\mathcal X$, when for all $x\in\mathcal X$ and $a\in\mathcal A$, we have

\vspace{-0.15in}
\begin{small}
\begin{equation*}
q_c^{(k+1)}(a|x) = \frac{\pi(a|x)\cdot \exp\big(Q_{q_c^{(k)}}(x,a)\big)}{\mathbb E_{a\sim\pi(\cdot|x)}\big[\exp\big(Q_{q_c^{(k)}}(x,a)\big)\big]},
\end{equation*}
\end{small}
\vspace{-0.15in}

and

\vspace{-0.15in}
\begin{small}
\begin{equation*}
Q_{q_c^{(k)}}(x,a) = \eta\cdot r(x,a) + \log\mathbb E_{x'\sim p(\cdot|x,a)}\big[\exp\big(V_{q_c^{(k)}}(x')\big)\big].
\end{equation*}
\end{small}
\vspace{-0.15in}

% \noindent{\textbf{Policy Evaluation}}:
% Consider the $q_c$-induced Bellman operator $\mathcal T_{q_c}$ in Eq.~\ref{eq:induced-operator1} and a state-action value function
% $V_0: \mathcal X \rightarrow \mathbb R$ with finite state and action spaces $\mathcal X\times\mathcal A$, and define $V_{k+1} = \mathcal T_{q_c}[V_k]$
% . Then the sequence $V_k$ will converge to the optimistic entropy-regularized value function of $q_c$, i.e., 
% \[
% \lim_{k\rightarrow\infty}V_k(x)= \max_{q_d\in\Delta_{\mathcal X}} \! \mathbb E\big[\sum_{t=0}^{T-1}\eta\cdot r(\obs_t,\control_t) - \log\frac{q_c(a_t|x_t)}{\pi(a_t|x_t)} - \log\frac{q_d(x_{t+1}|x_t,a_t)}{p(x_{t+1}|x_t,a_t)}\mid p_0,q_d,q_c, x_0=x\big].
% \]
% \begin{proof}
% Define the entropy augmented reward as $r_{q_c,\pi,p}(x, a,x') :=\eta\cdot r(x,a) - \log({q_c(a|x)}/{\pi(a|x)})- \log({q_d(x'|x,a)}/{p(x'|x,a)})$ and rewrite the update
% rule as
% \[
% V(x) \leftarrow \max_{q_d\in\Delta_{\mathcal X}}\mathbb E_{x'\sim q_d(\cdot|x,a),a\sim q_c(\cdot|x)}\left[r_{q_c,\pi,p}(x, a,x') + V(x')\right]
% \]
% and apply the standard convergence results for optimistic/robust policy evaluation \citep{iyengar2005robust}. The assumption of finite state and action spaces guarantees that both the entropy regularization terms on policies and dynamics are bounded.
% \end{proof}

\begin{proof}
Since from (\ref{eq:posterior-policy}), for all $x\in\mathcal X$ and $a\in\mathcal A$, we have

\vspace{-0.15in}
\begin{small}
\begin{equation*}
q_c^{(k+1)}(a|x) = \argmax_{q_c}\mathbb E_{a\sim q_c(\cdot|x)}\big[Q_{q_c^{(k)}}(x,a) - \log\frac{q_c(a|x)}{\pi(a|x)}\big],
\end{equation*}
\end{small}
\vspace{-0.15in}

for all $x\in\mathcal X$, we may write 

\vspace{-0.15in}
\begin{small}
\begin{align}
\mathbb E_{a\sim q_c^{(k+1)}(\cdot|x)}\big[Q_{q_c^{(k)}}(x,a) - \log\frac{q_c^{(k+1)}(a|x)}{\pi(a|x)}\big] &\geq \mathbb E_{a\sim q_c^{(k)}(\cdot|x)}\big[Q_{q_c^{(k)}}(x,a) - \log\frac{q_c^{(k)}(a|x)}{\pi(a|x)}\big] \nonumber \\
&\stackrel{\text{(a)}}{=} \mathcal T_{q_c^{(k)}}[V_{q_c^{(k)}}](x) = V_{q_c^{(k)}}(x).
\label{eq:temp0}
\end{align}
\end{small}
\vspace{-0.15in}

{\bf (a)} This is from (\ref{eq:induced-operator1}).

We know that if we start from any arbitrary value function $V$ and iteratively apply $\mathcal T_{q_c^{(k+1)}}$, we will convergence to $V_{q_c^{(k+1)}}$. If we start from $V=V_{q_c^{(k)}}$, for all $x\in\mathcal X$, we have

\vspace{-0.15in}
\begin{small}
\begin{align*}
V_{q_c^{(k+1)}}(x) &= \lim_{n\rightarrow\infty}\mathcal T^n_{q_c^{(k+1)}}[V_{q_c^{(k)}}](x) = \lim_{n\rightarrow\infty}\mathcal T^{n-1}_{q_c^{(k+1)}}\Big[\mathcal T_{q_c^{(k+1)}}[V_{q_c^{(k)}}]\Big](x) \\
&\stackrel{\text{(a)}}{=} \lim_{n\rightarrow\infty}\mathcal T^{n-1}_{q_c^{(k+1)}}\Big[\mathbb E_{a\sim q_c^{(k+1)}(\cdot|x)}\big[Q_{q_c^{(k)}}(x,a) - \log\frac{q_c^{(k+1)}(a|x)}{\pi(a|x)}\big]\Big] \\
&\stackrel{\text{(b)}}{\geq} \lim_{n\rightarrow\infty}\mathcal T^{n-1}_{q_c^{(k+1)}}[V_{q_c^{(k)}}](x) \geq \ldots \geq \mathcal T_{q_c^{(k+1)}}[V_{q_c^{(k)}}](x) \geq V_{q_c^{(k)}}(x).
\end{align*}
\end{small}
\vspace{-0.15in}

{\bf (a)} This is by replacing $\mathcal T_{q_c^{(k+1)}}[V_{q_c^{(k)}}](x)$ with $\mathbb E_{a\sim q_c^{(k+1)}(\cdot|x)}\big[Q_{q_c^{(k)}}(x,a) - \log\frac{q_c^{(k+1)}(a|x)}{\pi(a|x)}\big]$ from (\ref{eq:induced-operator1}).

{\bf (b)} This is from (\ref{eq:temp0}) and the monotonicity of the operator $\mathcal T_{q_c^{(k+1)}}$ from Lemma~\ref{lem:qc-induced-optimal-operators}. 

This concludes the proof.
\end{proof}

The policy improvement step of the model-based and model-free PI algorithms discussed in Section~\ref{subsec:PI-VI-Algos} perform the update of Eq.~\ref{eq:posterior-policy-closed-form}. Calculating the denominator of this update when the number of actions is large or infinite (continuous action space) could not be done efficiently. In this case, similar to a number of algorithms in the literature (e.g.,~soft actor-critic), we replace the update of Eq.~\ref{eq:posterior-policy-closed-form} with the following KL minimization:

\vspace{-0.15in}
\begin{small}
\begin{equation}
\label{eq:KL-PI-Improvement}
q_c^{(k+1)}(\cdot|x) = \argmin_{q_c\in\Delta_{\mathcal A}}\;\text{KL}\left(q_c(\cdot|x)\;||\;\frac{\pi(\cdot|x)\cdot \exp\big(Q_{q_c^{(k)}}(x,\cdot)\big)}{\mathbb E_{a\sim\pi(\cdot|x)}\big[\exp\big(Q_{q_c^{(k)}}(x,a)\big)\big]}\right), \quad \forall x\in\mathcal X.
\end{equation}
\end{small}
\vspace{-0.15in}

Now in the following corollary, we prove that even in this case we will see policy improvement, and thus, the algorithms will eventually converge to the optimal variational distributions $q^*=(q_c^*,q_d^*)$. 

\begin{corollary}
\label{corr:KL-PI-Improvement}
Let at iteration $k$, the variational policy, $q_c^{(k+1)}$, is computed as the exact solution to the KL optimization (\ref{eq:KL-PI-Improvement}). Then, we have $V_{q_c^{(k+1)}}(x)\geq V_{q_c^{(k)}}(x),\;\forall x\in\mathcal X$.
\end{corollary}

\begin{proof}
Since $q_c^{(k+1)}$ is the minimizer of (\ref{eq:KL-PI-Improvement}), we may write

\vspace{-0.15in}
\begin{small}
\begin{equation}
\label{eq:temp000}
\text{KL}\left(q_c^{(k+1)}(a|x)\;||\;\frac{\pi(a|x)\cdot \exp\big(Q_{q_c^{(k)}}(x,a)\big)}{\mathbb E_{a\sim\pi(\cdot|x)}\big[\exp\big(Q_{q_c^{(k)}}(x,a)\big)\big]}\right) \leq \text{KL}\left(q_c^{(k)}(a|x)\;||\;\frac{\pi(a|x)\cdot \exp\big(Q_{q_c^{(k)}}(x,a)\big)}{\mathbb E_{a\sim\pi(\cdot|x)}\big[\exp\big(Q_{q_c^{(k)}}(x,a)\big)\big]}\right).
\end{equation}
\end{small}
\vspace{-0.15in}

Let $Z_{q_c^{(k)}}(x) := E_{a\sim\pi(\cdot|x)}\big[\exp\big(Q_{q_c^{(k)}}(x,a)\big)\big]$, then we may rewrite (\ref{eq:temp000}) as

\vspace{-0.15in}
\begin{small}
\begin{align*}
\mathbb E_{a\sim q_c^{(k+1)}(\cdot|x)}\big[Q_{q_c^{(k)}}(x,a) &- \log\frac{q_c^{(k+1)}(a|x)}{\pi(a|x)}\big] - \mathbb E_{a\sim q_c^{(k+1)}}\big[Z_{q_c^{(k)}}(x)\big] \geq \\
&\mathbb E_{a\sim q_c^{(k)}(\cdot|x)}\big[Q_{q_c^{(k)}}(x,a) - \log\frac{q_c^{(k)}(a|x)}{\pi(a|x)}\big] - \mathbb E_{a\sim q_c^{(k)}}\big[Z_{q_c^{(k)}}(x)\big].
\end{align*}
\end{small}
\vspace{-0.15in}

Since $Z_{q_c^{(k)}}(\cdot)$ is a function of $x$, we have 

\vspace{-0.15in}
\begin{small}
\begin{align*}
\mathbb E_{a\sim q_c^{(k+1)}(\cdot|x)}\big[Q_{q_c^{(k)}}(x,a) &- \log\frac{q_c^{(k+1)}(a|x)}{\pi(a|x)}\big] - Z_{q_c^{(k)}}(x) \geq \\ 
&\mathbb E_{a\sim q_c^{(k)}(\cdot|x)}\big[Q_{q_c^{(k)}}(x,a) - \log\frac{q_c^{(k)}(a|x)}{\pi(a|x)}\big] - Z_{q_c^{(k)}}(x).
\end{align*}
\end{small}
\vspace{-0.15in}

Thus, the $Z$ terms are removed from both sides of the above inequality and we return to Eq.~\ref{eq:temp0} in the proof of Lemma~\ref{lem:policy_iteration}. The rest of the proof is similar to that of Lemma~\ref{lem:policy_iteration}. 
\end{proof}

%%%%%%%%%%%%%%%%%%%%%%%%%%%%%%%%%%%%%%%%%%%%%%%%%%%%%%%%%%%%%%%%%%%%%%%%%%
%%%%%%%%%%%%%%%%%%%%%%%%%%%%%%%%%%%%%%%%%%%%%%%%%%%%%%%%%%%%%%%%%%%%%%%%%%
%%%%%%%%%%%%%%%%%%%%%%%%%%%%%%%%%%%%%%%%%%%%%%%%%%%%%%%%%%%%%%%%%%%%%%%%%%
%%%%%%%%%%%%%%%%%%%%%%%%%%%%%%%%%%%%%%%%%%%%%%%%%%%%%%%%%%%%%%%%%%%%%%%%%%
%%%%%%%%%%%%%%%%%%%%%%%%%%%%%%%%%%%%%%%%%%%%%%%%%%%%%%%%%%%%%%%%%%%%%%%%%%

\newpage
\section{Pseudo-code of VMBPO}
\label{sec:VMBPO-algo}

This section contains the pseudo-code of our variational model-based policy optimization (VMBPO) algorithm, whose E-step and M-step have been described in details in Sections~\ref{subsec:VMBPO-E-step} and~\ref{subsec:VMBPO-M-step}.

\begin{algorithm}[H]
\begin{small}
\caption{Variational Model-based Policy Optimization (VMBPO)}
\label{alg:VMBPO}
\begin{algorithmic}[1]
\STATE {\bf Inputs}: replay buffer $\mathcal{D}$; $\;\;$ neural networks representing variational dynamics $\theta_d$, variational policy $\theta_c$, log-likelihood ratio $\theta_\nu$, value function $\theta_v$, action-value function $\theta_q$, target value function $\theta'_v$, target action-value function $\theta'_q$, baseline policy $\theta_\pi$;
\FOR{$t = 1,2,\ldots$}
\FOR{a number of interactions with the environment}
\STATE Observe state $x$; $\quad$ Take action $a \sim \pi(\cdot|x;\theta_\pi)$; $\quad$ Observe $r=r(x,a)\;\wedge\;x'\sim p(\cdot|x,a)$; 
\STATE Update the buffer $\mathcal{D} \leftarrow \mathcal{D} \cup (x,a,r,x^\prime)$; $\quad$ Replace $x\leftarrow x'$; 
\ENDFOR
\STATE {\color{gray}\# E-step $\qquad$ ($K$ is the number of E-step iterations)}
\FOR{$k = 1,\ldots,K$} 
\STATE {\color{gray}\# Step~1 $\qquad$ (updating variational dynamics $q_d$)}
\STATE Sample a number of $(x,a,r,x')\sim\mathcal D$; $\quad$ Update $q_d$ parameter $\theta_d$ using gradient of (\ref{eq:qd-update3});
\STATE {\color{gray}\# Step~2 $\qquad$ (updating log-likelihood ratio $\nu=\log(q_d/p)$)}
\STATE Sample a number of $(x,a,x')\sim\mathcal D$; $\quad$ Sample $x'\sim q_d$ for the same $(x,a)$;
\STATE Update $\nu$ parameter $\theta_\nu$ using gradient of (\ref{eq:nu_update-2});
\STATE {\color{gray}\# Step~3 $\qquad$ (critic update $\;$ -- $\;$ updating $V_{q_c}$ and $Q_{q_c}$)}
\STATE Sample a number of $(x,a,r,x')$ from the model; \hfill {\color{gray}\# $a\sim q_c(\cdot|x),\;x'\sim q_d(\cdot|x,a)$} 
\STATE Update $Q$ parameter $\theta_q$ using gradient of (\ref{eq:critic-Q-2});
\STATE Sample a number of $(x,a)$ from the model; \hfill {\color{gray}\# $a\sim q_c(\cdot|x)$}
\STATE Update $V$ parameter $\theta_v$ using gradient of (\ref{eq:critic-V-2});
\STATE {\color{gray}\# Step~4 $\qquad$ (actor update $\;$ -- $\;$ updating $q_c$ $\;$ -- $\;$ policy improvement)}
\STATE Update $q_c$ parameter $\theta_c$ either using gradient of (\ref{eq:actor-update-2}) or by solving (\ref{eq:posterior-policy-closed-form}) in closed-form; 
\STATE {\color{gray}\# target networks $\theta_v',\theta'_q$ are set to an exponentially moving average of the value networks $\theta_v,\theta_q$}
\STATE $\theta'_v \leftarrow \tau\theta_v + (1-\tau)\theta'_v$; $\quad$ $\theta'_q \leftarrow \tau\theta_q + (1-\tau)\theta'_q$; $\quad$ 
\ENDFOR
\STATE {\color{gray}\# M-step $\qquad$ (updating the baseline policy $\pi$)}
\STATE Update baseline policy $\pi$ parameter $\theta_\pi$ either by setting it to $\theta_c$ or by solving the MAP problem (\ref{eq:M-step-Opt2})
\ENDFOR
\end{algorithmic}
\end{small}
\end{algorithm}

%%%%%%%%%%%%%%%%%%%%%%%%%%%%%%%%%%%%%%%%%%%%%%%%%%%%%%%%%%%%%%%%%%%%%%%%%%
%%%%%%%%%%%%%%%%%%%%%%%%%%%%%%%%%%%%%%%%%%%%%%%%%%%%%%%%%%%%%%%%%%%%%%%%%%
%%%%%%%%%%%%%%%%%%%%%%%%%%%%%%%%%%%%%%%%%%%%%%%%%%%%%%%%%%%%%%%%%%%%%%%%%%
%%%%%%%%%%%%%%%%%%%%%%%%%%%%%%%%%%%%%%%%%%%%%%%%%%%%%%%%%%%%%%%%%%%%%%%%%%
%%%%%%%%%%%%%%%%%%%%%%%%%%%%%%%%%%%%%%%%%%%%%%%%%%%%%%%%%%%%%%%%%%%%%%%%%%

\newpage
\section{E-step with Model-free Policy Iteration}
\label{sec:model-free-E-step-Algo}
In Sec.~\ref{subsec:PI-VI-Algos}, we described a model-free PI algorithm for solving the E-step of our variational formulation. In this algorithm, $q_d$ is computed at the end of the E-step, and thus, only used to generate samples in the M-step. We call the RL-version of this algorithm, whose E-step is model-free and M-step is model-based, VMBPO with model-free E-step (VMBPO-MFE). In VMBPO-MFE, we first estimate $Q_{q_c}$ using (\ref{eq:Q-function}) and then approximate $V_{q_c}$ by minimizing the (fixed-point) loss $(V-\mathcal T_{q_c}V)^2$, where $\mathcal T_{q_c}$ is computed from (\ref{eq:induced-operator1}) by setting $Q$ equal to the target action-value network $Q'$. The $q_c$ update (policy improvement) is exactly the same as the actor update (Step~4) of VMBPO, described in Sec.~\ref{subsec:VMBPO-E-step}. After several E-step updates, $q_d$ is computed from (\ref{eq:posterior-dynamics-closed-form}). This is followed by the M-step, which is identical to that of VMBPO, described in Sec.~\ref{subsec:VMBPO-M-step}. We report the details of VMBPO-MFE and its pseudo-code in Appendix~\ref{sec:model-free-E-step-Algo}. Although VMBPO-MFE is less complex than VMBPO, our experiments in Sec.~\ref{sec:experiments} show that it is less sample efficient (i.e.,~achieves worse performance than VMBPO with the same number of real samples). The main reason for this is that VMBPO uses simulated samples in both E and M steps, while VMBPO-MFE only uses them in the M-step. Moreover, our experiments show that VMBPO-MFE may degenerate in certain cases, due to the instability of the exponential temporal difference (TD) learning in the critic step. The $\log$ expectation term in (\ref{eq:Q-function}) creates challenges for finding an unbiased empirical loss to estimate $Q$, and when it is removed by taking exponential from both sides of (\ref{eq:Q-function}), the resulting exponential terms cause numerical instability in the updates.

    \begin{algorithm}[H]
\begin{small}
\caption{Variational Model-based Policy Optimization with Model-free E-step (VMBPO-MFE)}
\label{alg:VMBPO_MFE}
\begin{algorithmic}[1]
\STATE {\bf Inputs}: replay buffer $\mathcal{D}$; $\;\;$ neural networks representing variational policy $\theta_c$, value function $\theta_v$, action-value function $\theta_q$, target value function $\theta'_v$, target action-value function $\theta'_q$, baseline policy $\theta_\pi$;
\FOR{$t = 1,2,\ldots$}
\FOR{a number of interactions with the environment}
\STATE Observe state $x$; $\quad$ Take action $a \sim \pi(\cdot|x;\theta_\pi)$; $\quad$ Observe $r=r(x,a)\;\wedge\;x'\sim p(\cdot|x,a)$; 
\STATE Update the buffer $\mathcal{D} \leftarrow \mathcal{D} \cup (x,a,r,x^\prime)$; $\quad$ Replace $x\leftarrow x'$; 
\ENDFOR
\STATE {\color{gray}\# E-step $\qquad$ ($K$ is the number of E-step iterations)}
\FOR{$k = 1,\ldots,K$} 
\STATE {\color{gray}\# Step~1 $\qquad$ (critic update $\;$ -- $\;$ updating $V_{q_c}$ and $Q_{q_c}$)}
\STATE Sample a number of $(x,a,r,x')$ from $\mathcal D$; \hfill {\color{gray}\# $a\sim q_c(\cdot|x),\;x'\sim p(\cdot|x,a)$} 
\STATE Update $Q$ parameter $\theta_q$ using gradient of (\ref{eq:vmbpo_mfe_critic_update_q});
\STATE Sample a number of $(x,a)$ from the $\mathcal D$; \hfill {\color{gray}\# $a\sim q_c(\cdot|x)$}
\STATE Update $V$ parameter $\theta_v$ using gradient of (\ref{eq:vmbpo_mfe_critic_update_v});
\STATE {\color{gray}\# Step~2 $\qquad$ (actor update $\;$ -- $\;$ updating $q_c$ $\;$ -- $\;$ policy improvement)}
\STATE Update $q_c$ parameter $\theta_c$ either using gradient of (\ref{eq:actor-update-2}) or by solving (\ref{eq:posterior-policy-closed-form}) in closed-form; 
\STATE {\color{gray}\# target networks $\theta_v',\theta'_q$ are set to an exponentially moving average of the value networks $\theta_v,\theta_q$}
\STATE $\theta'_v \leftarrow \tau\theta_v + (1-\tau)\theta'_v$; $\quad$ $\theta'_q \leftarrow \tau\theta_q + (1-\tau)\theta'_q$; $\quad$ 
\ENDFOR
\STATE {\color{gray}\# M-step $\qquad$ (updating the baseline policy $\pi$)}
\STATE Update baseline policy $\pi$ parameter $\theta_\pi$ either by setting it to $\theta_c$
\ENDFOR
\end{algorithmic}
\end{small}
\end{algorithm}

\subsection{The Model-free E-step Critic Update}\label{sec:vmbpo_mfe_critic_update}
Suppose we have access to the policy $\pi$, and posterior policy $q_c$, we now aim to learn the corresponding value functions (critic) $V_{\pi,q_c}$ and $Q_{\pi,q_c}$. Recall from Lemma \ref{lem:optimal-operator1}, we know that $V_{\pi,q_c}$ is a unique solution of fixed-point equation $\mathcal T_{q_c}[V](x)=V(x)$, $\forall x\in\mathcal X$. Suppose we parameterize $V_{\pi,q_c}(x)$ with function approximation $\widehat V_{\pi,q_c}(x;\kappa)$, and similarly $Q_{\pi,q_c}(x,a)$ with $\widehat Q_{\pi,q_c}(x,a;\xi)$. Similar to soft DQN, one way to learn $\widehat V_{\pi,q_c}(x;\kappa)$ and $\widehat Q_{\pi,q_c}(x,a;\xi)$ is by minimizing the following objective function respectively, over the data from the replay buffer $\mathcal D$ sampled from the environment:
\begin{small}
\begin{equation}\label{eq:vmbpo_mfe_critic_update_q}
\theta^*_q\leftarrow\arg\min_\theta \sum_{(x, a, r, x')\sim\mathcal D}\!\!\left(\!\exp\left(\widehat Q_{\pi,q_c}(x,a;\theta)\!-\!\eta r\!-\widehat V_{\pi,q_c}(x';\theta_v^{\prime}\!)\right)- 1\right)^2, 
\end{equation}
\end{small}
where $V_{\pi,q_c}(x';\theta_v^{\prime})$ is a \emph{target} Q-function, and
\begin{small}
\begin{equation}\label{eq:vmbpo_mfe_critic_update_v}
\theta^*_v\leftarrow\arg\min_{\theta} \sum_{(x, a, r, x')\sim\mathcal D}\left(\widehat V_{\pi,q_c}(x;\theta) - \int_{\control\in\mathcal{A}}q_c(\control|\obs)\left(\widehat Q_{\pi,q_c}(x,a;\theta_q) - \log\frac{q_c(\control|\obs)}{\pi(\control|\obs)}\right)\right)^2,
\end{equation}
\end{small}

The above learning is completely \emph{off-policy}---the target is valid no matter how the experience was generated (as long as it
is sufficiently exploratory). Under this loss, critic learning can be viewed as $\ell_2$-regression of $\exp(\widehat Q_{\pi,q_c}(x,a;\theta_q)-\eta r-V_{\pi,q_c}(x;\theta_v))$ w.r.t. the target label $1$, such that the value function $V_{\pi,q_c}(x;\theta_v)$ is learned to minimize the the \emph{mean squared Bellman error}: $(\mathcal T_{q_c}[V](x)-V(x))^2$ and to enforce $\widehat Q_{\pi,q_c}(x,a;\theta_q)=\eta\cdot r(\obs,\control)+\log\int_{x'\in\mathcal X}P(x'|x,a)\exp \widehat V_{\pi,q_c}(x';\theta_v)$.

%%%%%%%%%%%%%%%%%%%%%%%%%%%%%%%%%%%%%%%%%%%%%%%%%%%%%%%%%%%%%%%%%%%%%%%%%%
%%%%%%%%%%%%%%%%%%%%%%%%%%%%%%%%%%%%%%%%%%%%%%%%%%%%%%%%%%%%%%%%%%%%%%%%%%
%%%%%%%%%%%%%%%%%%%%%%%%%%%%%%%%%%%%%%%%%%%%%%%%%%%%%%%%%%%%%%%%%%%%%%%%%%
%%%%%%%%%%%%%%%%%%%%%%%%%%%%%%%%%%%%%%%%%%%%%%%%%%%%%%%%%%%%%%%%%%%%%%%%%%
%%%%%%%%%%%%%%%%%%%%%%%%%%%%%%%%%%%%%%%%%%%%%%%%%%%%%%%%%%%%%%%%%%%%%%%%%%
%%%%%%%%%%%%%%%%%%%%%%%%%%%%%%%%%%%%%%%%%%%%%%%%%%%%%%%%%%%%%%%%%%%%%%%%%%
%%%%%%%%%%%%%%%%%%%%%%%%%%%%%%%%%%%%%%%%%%%%%%%%%%%%%%%%%%%%%%%%%%%%%%%%%%
%%%%%%%%%%%%%%%%%%%%%%%%%%%%%%%%%%%%%%%%%%%%%%%%%%%%%%%%%%%%%%%%%%%%%%%%%%
%%%%%%%%%%%%%%%%%%%%%%%%%%%%%%%%%%%%%%%%%%%%%%%%%%%%%%%%%%%%%%%%%%%%%%%%%%
%%%%%%%%%%%%%%%%%%%%%%%%%%%%%%%%%%%%%%%%%%%%%%%%%%%%%%%%%%%%%%%%%%%%%%%%%%

\newpage
\section{Experimental Details}\label{appendix:experimental_details}
\begin{table}[h]
\centering
\begin{tabular}{|l|c|c|c|}
\hline
\textbf{Environment} & \textbf{State dimension} & \textbf{Action dimension} \\ [0.5ex]
\hline
\hline
Pendulum & 3 & 1 \\
\hline
Reacher & 11 & 2 \\
\hline
Hopper & 11 & 3 \\
\hline
Reacher7DoF & 14 & 7 \\
\hline
Walker2D & 17 & 6 \\
\hline
HalfCheetah & 17 & 6 \\
\hline
\end{tabular}
\caption{State and Action dimensions of various benchmark environments.}
\label{table:benchmark}
\end{table}

\begin{table}[h]
\centering
\begin{tabular}{|l|c|}
\hline
\textbf{Hyper Parameters for MBPO and VMBPO} & \textbf{Value(s)} \\ [0.5ex]
\hline
\hline
Discount Factor & 0.99 \\
\hline
Number of Model Ensemble Networks & 7 \\
\hline
Number of Expert Networks & 2 \\
\hline
Number of Q Ensemble Networks & 2 \\
\hline
Dynamics Model Network Architecture & MLP with 4 hidden layers of size 200 \\
\hline
Critic Network Architecture & MLP with 2 hidden layers of size 200 \\
\hline
Actor Network Architecture & MLP with 2 hidden layers of size 200 \\
\hline
Exploration policy & $\mathcal{N}(0, \sigma=1)$ \\
\hline
Exploration noise ($\sigma$) decay & 0.999 \\
\hline
Exploration noise ($\sigma$) minimum & 0.025 \\
\hline
Temperature & 0.99995 \\
\hline
Soft target update rate ($\tau$) & 0.005 \\
\hline
Replay memory size (Both $\mathcal D$, $\mathcal E$) & $10^6$ \\
\hline
Mini-batch size (AC) & 64 \\
\hline
Mini-batch size (Model-learning) & 256 \\
\hline
Model learning rate & 0.0003 \\
\hline
Critic learning rates & 0.001, 0.0005, 0.0002 \\
\hline
Actor learning rates & 0.0005, 0.0002, 0.0001 \\
\hline
Neural network optimizer & Adam \\
\hline
\end{tabular}
\caption{Hyper parameters settings for MBPO and VMBPO. 
         We sweep over the critic learning rates and actor learning rates for tuning.}
\label{table:hyper_params}
\end{table}

For baseline algorithms, we either use the code the authors open-sourced or implement on our own and deliberately use the configurations shown in the literature. Among them, Table~\ref{table:hyper_params} shows the hyper parameters for MBPO and VMBPO in more detail.

\subsection{Additional Experimental Results}\label{appendix:additional_exp}
\begin{table}[th!]
\begin{adjustwidth}{-.5in}{-.5in}
\centering
\scalebox{0.7}{
\begin{tabular}{|l|c|c|c|c|c|c|c|}
\hline
\textbf{Env.} & \textbf{VMBPO} & \textbf{MBPO} & \textbf{STEVE} & \textbf{PETS} & \textbf{VMBPO-MFE} & \textbf{SAC} & \textbf{MPO} \\ [0.5ex]
\hline
\hline
Pendulum & -147.4 $\pm$ 94.1 & \textbf{-146.8} $\pm$ 272.6 & --- & --- & -511.9 $\pm$ 384.4 & \textbf{-146.8} $\pm$ 450.6 & -605.2 $\pm$ 389.6  \\
\hline
Hopper & \textbf{2137.2} $\pm$ 1016.6 & 1689.5 $\pm$ 934.5 & --- & --- & 485.4 $\pm$ 389.3 & 1262.2 $\pm$ 803.3 & 780.8 $\pm$ 629.6 \\
\hline
Walker2D & \textbf{2817.6} $\pm$ 1076.1 & 2356.4 $\pm$ 1104.3 & --- & --- & 1447.1 $\pm$ 767.1 & 1341.6 $\pm$ 1092.6 & 1590.3 $\pm$ 860.7 \\
\hline
HalfCheetah & \textbf{8644.6} $\pm$ 3291.1 & 7573.4 $\pm$ 4056.9 & --- & --- & 2834.6 $\pm$ 1062.9 & 6312.0 $\pm$ 2299.7 & 3258.2 $\pm$ 970.1 \\
\hline
Reacher & \textbf{-13.5} $\pm$ 38.7 & -17.5 $\pm$ 44.8 & --- & --- & -122.2 $\pm$ 507.0 & -77.2 $\pm$ 50.6 & -168.2 $\pm$ 477.1 \\
\hline
Reacher7DoF & \textbf{-15.2} $\pm$ 66.4 & -17.2 $\pm$ 101.6 & --- & --- & -78.9 $\pm$ 439.1 & -114.2 $\pm$ 196.9 & -93.8 $\pm$ 426.9\\
\hline
\end{tabular}
}
\end{adjustwidth}
\caption{The mean $\pm$ SD of final returns over all hyper-parameter configurations. VMBPO is more robust to hyper-parameter configurations than other baselines. We do not include PETS and STEVE because the hyper-parameter configurations are directly adopted from their papers.}
\label{table:exp1_all_mean}
\end{table}

\begin{figure}[th!]
\centering
\includegraphics[width=\textwidth]{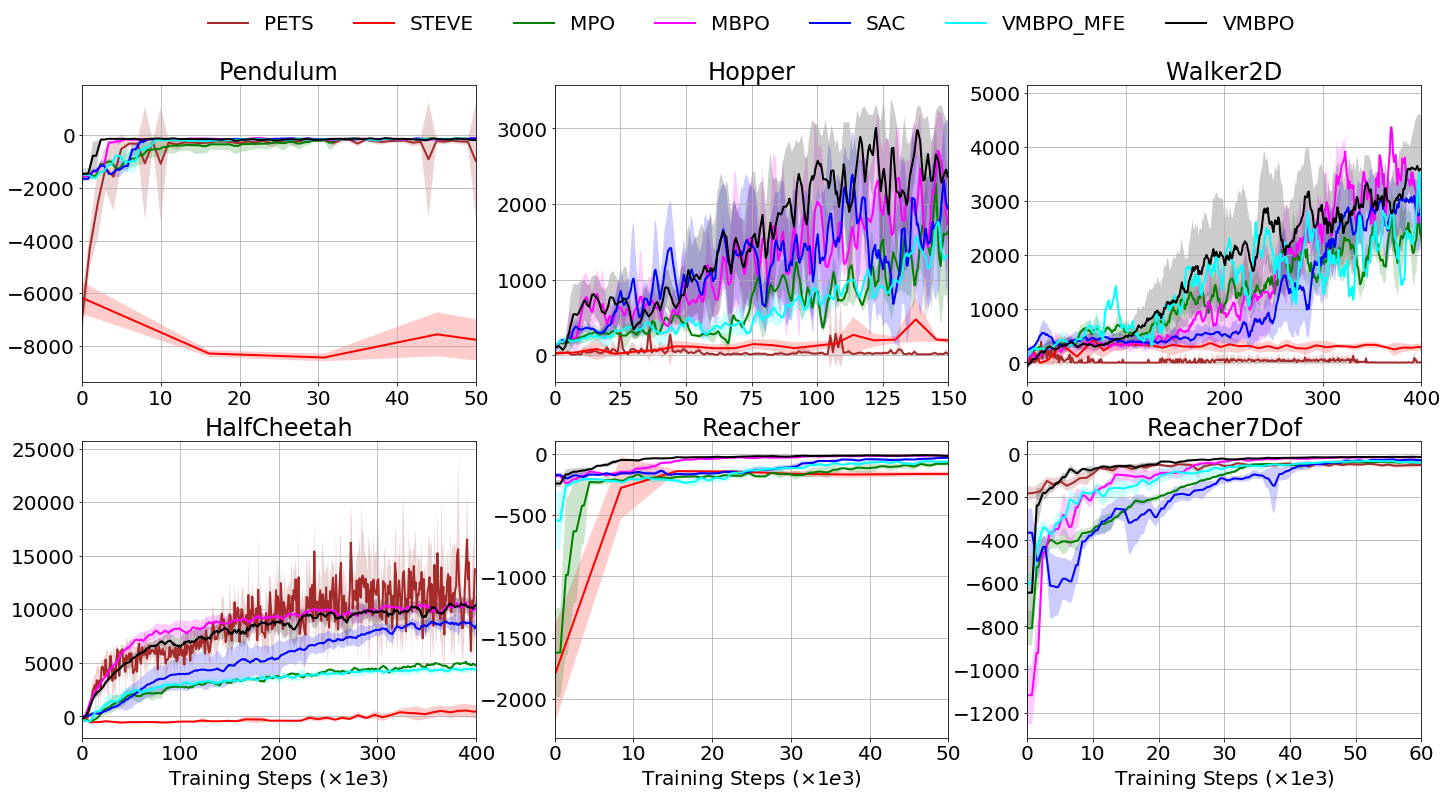}
\caption{Mean cumulative reward of the best hyper parameter configuration over 5 random seeds.}
\label{fig:exp1_best_mean}
\end{figure}

\end{document}